\documentclass[reqno]{amsart}
\usepackage{xcolor}
\usepackage{graphicx}
\theoremstyle{plain}
\begingroup
 
\newtheorem{theorem}{Theorem} 

\newtheorem{lemma}{Lemma}
\endgroup
\newcommand{\nwc}{\newcommand}
\nwc{\qref}[1]{(\ref{#1})}
\nwc{\cadlag}{c\`{a}dl\`{a}g}
\nwc{\la}{\label}
\nwc{\nn}{\nonumber}
\nwc{\Z}{\mathbb{Z}}
\nwc{\C}{\mathbb{C}}
\nwc{\T}{\mathbb{T}}
\nwc{\E}{\mathbb{E}}
\nwc{\R}{\mathbb{R}}
\nwc{\N}{\mathbb{N}}
\nwc{\Rn}{\mathbb{R}^n}
\nwc{\PP}{\mathcal{P}}
\nwc{\M}{\mathcal{M}}
\nwc{\Ito}{It\^{o}}
\nwc{\DiffM}{\mathrm{Diff}(M)}
\nwc{\orbit}{\mathcal{O}}
\nwc{\bbb}{\mathbf{M}}

\nwc{\law}{\stackrel{\mathcal{L}}{\rightarrow}}
\nwc{\eqd}{\stackrel{\mathcal{L}}{=}}
\nwc{\vp}{\varphi}
\nwc{\veps}{\varepsilon}
\nwc{\eps}{\veps}
\nwc{\dnto}{\downarrow}
\nwc{\nsup}{^{(n)}}
\nwc{\ksup}{^{(k)}}
\nwc{\jsup}{^{(j)}}
\nwc{\nksup}{^{(n_k)}}
\nwc{\inv}{^{-1}}
\nwc{\argmin}{\mathrm{argmin}}
\nwc{\argmax}{\mathrm{argmax}}
\nwc{\Tr}{\mathrm{Tr}}
\nwc{\Id}{\mathrm{Id}}
\nwc{\Pn}{\mathbb{P}(n)}
\nwc{\PnR}{\mathbb{P}(n;\mathbb{R})}
\nwc{\PnC}{\mathbb{P}(n;\mathbb{C})}
\nwc{\Hn}{\mathbb{H}(n)}
\nwc{\HnR}{\mathbb{H}(n;\mathbb{R})}
\nwc{\HnC}{\mathbb{H}(n;\mathbb{C})}
\nwc{\An}{\mathbb{A}(n)}
\nwc{\AnR}{\mathbb{A}(n;\mathbb{R})}
\nwc{\AnC}{\mathbb{A}(n;\mathbb{C})}
\nwc{\Mn}{\mathbb{M}(n)}
\nwc{\Md}{\mathbb{M}_d}
\nwc{\Mm}{\mathbb{M}_m}
\nwc{\feasible}{\mathcal{F}}
\nwc{\GLn}{GL(n)}
\nwc{\GLnC}{GL(n;\mathbb{C})}
\nwc{\GLnR}{GL(n;\mathbb{R})}
\nwc{\Poincare}{Poincar\'{e}}
\nwc{\Un}{U(n)}
\nwc{\Sn}{\mathbb{S}^{n}}
\nwc{\Tn}{\mathbb{T}^n}
\nwc{\ev}{\mathrm{ev}}
\nwc{\Diff}{\mathrm{Diff}}
\nwc{\orbitx}{\mathcal{O}_W}
\nwc{\volume}{\mathrm{vol}}
\nwc{\symm}{\mathrm{Symm}}
\nwc{\psd}{\mathbb{P}}
\nwc{\od}{O_d}
\nwc{\balance}{\mathcal{M}}
\nwc{\manmap}{\mathfrak{x}}
\nwc{\manmapbeta}{\mathfrak{y}}
\nwc{\grad}{\mathrm{grad}}
\nwc{\ww}{\mathbf{W}}
\nwc{\ff}{\mathcal{F}}
\nwc{\ggg}{\mathcal{G}}
\nwc{\hh}{\mathcal{H}}
\nwc{\ts}{n}
\nwc{\anti}{\mathbb{A}}
\nwc{\diag}{\mathrm{diag}}
\nwc{\Lojas}{Lojasiewicz}
\nwc{\bd}{\mathbf}

\nwc{\Mr}{\mathfrak{M}_r}
\nwc{\Mdd}{\mathfrak{M}_d}
\nwc{\gbw}{g^{BW}}
\nwc{\aaa}{\mathcal{A}}
\nwc{\Stief}{\mathrm{St}}
\nwc{\Fr}{\iota}
\nwc{\weyl}{\mathcal{W}}
\nwc{\Her}{\mathrm{Her}}
\nwc{\refspace}{\mathcal{M}}
\nwc{\refgroup}{\mathcal{G}}
\nwc{\quotientspace}{\mathcal{X}}
\nwc{\refmetric}{g}
\nwc{\quotientmetric}{h}

\nwc{\Levy}{L\'{e}vy}
\nwc{\HC}{Harish-Chandra}
\nwc{\BW}{Bures-Wasserstein}
\nwc{\CH}{Cartan-Hadamard}
\nwc{\Hess}{\mathrm{Hess}}
\nwc{\Gunther}{\mathrm{G\"{u}nther}}

\theoremstyle{definition}
\newtheorem{defn}[theorem]{Definition} 
\newtheorem{remark}[theorem]{Remark}

\theoremstyle{remark}
 
\numberwithin{equation}{section}
\numberwithin{figure}{section}

\begin{document}

\title{The geometry of the Deep Linear Network}

\begin{abstract}
This article provides an expository account of training dynamics in the Deep Linear Network (DLN) from the perspective of the geometric theory of dynamical systems. 
Rigorous results by several authors are unified into a thermodynamic framework for deep learning. 

The analysis begins with a characterization of the invariant manifolds and Riemannian geometry in the DLN. This is followed by exact formulas for a Boltzmann entropy, as well as stochastic gradient descent of free energy using a Riemannian Langevin Equation. Several links between the DLN and other areas of mathematics are discussed, along with some open questions.
\end{abstract}

\author{Govind Menon}

\address{Division of Applied Mathematics, Brown University, 182 George St., Providence, RI 02912.}

\email{govind\_menon@brown.edu}
\thanks{This work was supported by the National Science Foundation (DMS grants 171487, 2107205 and  2407055), the Simons Foundation (Award 561041) and the Charles Simonyi Foundation. The author thanks the School of Mathematics at the Institute for Advanced Study, Princeton for partial support during the completion of this work.}

\maketitle
\tableofcontents

\section{Introduction}
\label{sec:intro}
In its simplest form, deep learning is a version of function approximation by neural networks, where the parameters of the network are determined by given data. The best fit is determined, or the network is trained, by minimizing an empirical cost function using gradient descent. Many of the mysteries of deep learning concern training dynamics: Do we have convergence? If so, how fast and to what? How do we make training more efficient? How do the training dynamics depend on the network architecture or on the size of data?

This article focuses on mathematical foundations. Our goal is to illustrate the utility of the geometric theory of dynamical systems for the study of these questions. While gradient flows have been studied in mathematics since the 1930s, gradient flows arising in deep learning have two subtle aspects  --{\em overparametrization\/} and {\em degenerate loss functions\/} -- that prevent naive applications of the standard theory of gradient flows (see Section~\ref{subsec:overparam}) . We present a geometric framework for a simplified model, the Deep Linear Network (DLN), where these aspects can be studied with complete rigor.

The DLN is deep learning for {\em linear\/} functions. This reduces the training dynamics to gradient flows on spaces of matrices. Despite its apparent simplicity, the model has a rich mathematical structure. Overparametrization provides a foliation of phase space by invariant manifolds. Of these, there is a fundamental class of invariant manifolds, the {\em balanced manifolds\/}, which are themselves foliated by group orbits. This geometric structure allows us to define a natural Boltzmann entropy (the logarithm of the volume of a group orbit) that may be computed explicitly. Microscopic fluctuations that underlie the entropy may be described by Riemannian Langevin equations. This approach unifies the work of several authors into a thermodynamic framework. In particular, it suggests an entropic origin for implicit regularization.

My view is that the DLN is a gift to mathematics from computer science. On one hand, it is subtle, but tractable, providing a rich set of practical questions and insights. On the other hand, the study of the DLN is filled with sharp theorems, exact formulas and unexpected mathematical structure. While the gradient dynamics of the DLN are different from the standard theory, we also see familiar aspects in surprising combinations. This gives the analysis a classical feel, even though all the results here were obtained in the very recent past. There is plenty more to be discovered and the real purpose of this article is to explain why.

In order to apply the methods of dynamical systems theory, what is of most value is to understand the dynamicists `way of seeing'. This is not so much a collection of theorems, as a systematic use of geometry, and particular examples, to figure out what questions one should ask. Geometric methods provide a powerful intuition that is often a source of new discoveries. 

This is the approach we use in this article. All the theorems we prove are guided by the work of computer scientists in the area.  While this article does include some advanced mathematics, especially Riemannian geometry and random matrix theory, we stress explicit calculations, heuristic insights and representative examples. We hope this approach reveals the conceptual power of dynamical systems theory while remaining of interest to practitioners. The references are representative, not exhaustive, since our aim in this article is to provide a pedagogical treatment. We include a brief discussion of the literature on the DLN in Section~\ref{subsec:dln-dl}. For broader surveys of deep learning that include related mathematical ideas, the reader is referred to the recent books~\cite{Arora-book,Hanin-book}. 

The article concludes with open questions that emerge from this perspective. These include specific mathematical questions on the DLN, as well as the extension of our entropy formula to gauge groups arising in deep learning.

\section{A caricature of deep learning}
\label{sec:deep-learning}
In its simplest mathematical variant, the purpose of deep learning is to create a function $f: \R^{d_x} \to \R^{d_y}$ given a collection of training data 
$\{(X_k,Y_k)\}_{k=1}^\ts$, with $X_k\in \R^{d_x}$, $Y_k \in \R^{d_y}$ and $\ts$ denoting the size of the training dat set. In contrast with approximation by a linear combination of basis functions, such as the use of  Fourier series or polynomials, in deep learning a function is approximated using neural networks. The neural network consists of individual function elements (neurons) and a hierarchical architecture (the network) so that complicated functions can be constructed by scaling, translating and composing sums of the basic function element. The parameters, denoted for convenience $\mathbf{W} \in \R^M$, of the neural network describe the underlying scaling and translation (see~\cite[p.5]{DeVore-Hanin} for explicit descriptions).

  The parameters $\ww\in \R^M$ are determined from the data by minimizing an empirical loss function. A natural example is the quadratic loss function
\begin{equation}
\label{eq:empirical-loss}
L(\mathbf{W}) = \frac{1}{\ts}\sum_{k=1}^\ts |f(X_k;\mathbf{W})-Y_k|^2.   
\end{equation} 
The parameters $\mathbf{W}$ are then obtained by minimizing the loss function using gradient descent. We model this process with the differential equation
\begin{equation}
\label{eq:big-gradient}
    \dot{\mathbf{W}} = -\nabla_{\mathbf{W}} L(\mathbf{W}),
\end{equation}
where $\nabla_\mathbf{W}$ denotes the gradient with respect to the Euclidean norm in $\R^M$. In practice, the gradient descent is discrete in time, relies on the back-propagation algorithm, and includes random batch processing. These dynamics involve subtle corrections to the gradient flow~\eqref{eq:big-gradient}, but for the purposes of rigorous analysis, it is is first necessary to have an understanding of equation~\eqref{eq:big-gradient}.

Our aim in this article will be to foster a geometric understanding of training dynamics for deep networks. In order to do so, we first form a geometric picture of the underlying approximation theory. We will then use the DLN to make precise several aspects of this geometric picture.  Here are the important ideas in our work.
\begin{enumerate}
\item {\em Riemannian submersion.\/}
The creation of a function $\ww \mapsto f(\cdot,\cdot ;\ww)$ by the neural network defines a map $\ff:\R^M \to C(\R^{d_x},\R^{d_y})$. The image of this map, $\mathcal{G}$, is the class of functions that can be created with the network. We will visualize this by thinking of the parameter space $\R^M$ as living `upstairs' and its image $\mathcal{G}$ living `downstairs'. The natural geometric relation between these spaces is provided by Riemannian submersion. 

\item {\em Scale-by-scale composition\/}. Simple architectural motifs define the architecture for bump functions~\cite{Lapedes1988,GM-pt}. Thus, intuitively deep networks form complicated functions through the scale-by-scale composition of bump functions. The expressivity of the deep network -- its ability to approximate complicated functions -- is quantified by the `size' of $\mathcal{G} \subset C(\R^{d_x},\R^{d_y})$. The word size is in quotes, because of the complexity of the map $\ff$ and the fact that $C(\R^{d_x},\R^{d_y})$ is infinite-dimensional. However, this idea may be explored with approximation theory and numerical experiments~\cite{DeVore-Hanin}. 

\item {\em The geometry of overparametrization\/}. The problem is overparametrized when the same function $h\in \ggg$ may be created with different choices of weights $\ww \in \R^M$. Precisely, given $h \in \ggg$ we must understand a `large' inverse image $\ff^{-1}(h) :=\hh \subset\R^M $. It is natural to assume at first that $\hh$ is a manifold. Then $T_\ww \hh$ defines a set of `null directions' for the gradient descent. (Since the loss function $E$ depends on $\ww$ through $h=\ff(\ww)$ alone it does not change when $\ww$ is varied along the directions in $T_\ww \hh$.) If $\hh$ is large, then the gradient descent is restricted to a `thin' set of directions in $\R^M$. We will quantify these ideas precisely for the DLN with a Boltzmann entropy. Further, we show that noise in the null directions, gives rise to motion by curvature.
\item {\em Degenerate loss functions.\/}
While the loss function is quadratic, there may be many choices of $h$ such that the empirical loss function is minimized for such $h$ (for example, when $\ts$ is sufficiently small). Naively, this corresponds to {\em overfitting\/}. For example, sufficiently high-degree polynomials will fit the data, but also generate spurious oscillations in between the data points. The functions created by deep learning do not overfit. Heuristically, this seems to be due to the efficient coding of functions made possible by scale-by-scale composition. However, a rigorous analysis must explain this behavior as a consequence of the training dynamics. 
\end{enumerate}

We revisit these ideas at the conclusion of this article. This allows us to compare the DLN and deep learning equipped with a precise geometric understanding of training dynamics in the DLN.

\section{The deep linear network}
\label{sec:dln}

\subsection{The model}
Given a positive integer $m$, we denote by $\Mm$ the space of $m\times m$ real matrices; $\symm_m \subset \Mm$ the space of symmetric matrices; $\anti_m$ the space of anti-symmetric matrices; and $\psd_m \subset \Mm$ the space of positive semidefinite (psd) matrices. Finally, $O_m$ denotes the orthogonal group of dimension $m$ and $\Stief_{r,m}$ denotes the Stiefel manifold of $r$-frames in $\R^m$.

We fix two positive integer $d$ and $N$ referred to as the width and depth of the network. The state space for the DLN is $\Md^N$. Each point  $\mathbf{W}\in \Md^N$ is denoted by
\begin{equation}
    \label{eq:state-space}
    \mathbf{W}= (W_N,W_{N-1}, \ldots, W_1).
\end{equation}
We equip $\Md$ with the Frobenius norm so that $\Md^N$  is Euclidean with the norm 
\begin{equation}
    \label{eq:frobenius}
    \|\mathbf{W}\|^2 = \sum_{p=1}^N \Tr \left(W_p^T W_p\right).
\end{equation}
We define the projection $\phi: \Md^N \to \Md$ and the {\em end-to-end\/} matrix $W$ by\footnote{It is not necessary to assume that all the matrices are $d\times d$. All that is required is that the matrix multiplication in equation~\eqref{eq:state-space2} is well-defined. However, we restrict attention to square matrices to illustrate the main ideas.} 
\begin{equation}
    \label{eq:state-space2}
    \phi(\mathbf{W}) = W_N W_{N-1}\cdots W_1 =:W.
\end{equation}

We assume given an energy $E: \Md\to \R$. 
The training dynamics are described by the gradient flow in $\Md^N$ with respect to the Frobenius norm of the `lifted' loss function $L = E\circ \phi$
\begin{equation}
    \label{eq:big-grad1a}
    \dot{\mathbf{W}}= - \nabla_{\mathbf{W}} L(\mathbf{W}).
\end{equation}
This is a collection of $N$ equations in $\Md$ 
\begin{equation}
    \label{eq:big-grad1}
    \dot{W}_p= - \nabla_{W_p} E(W_N \cdots W_1), \quad p=1,\ldots,N.
\end{equation}
A computation using equation~\eqref{eq:state-space2} (see Section~\ref{subsec:chain-rule}) simplifies equation~\eqref{eq:big-grad1} to 
\begin{equation}
    \label{eq:big-grad2}
    \dot{W}_p= - (W_N\cdots W_{p+1})^T E'(W) (W_{p-1}\cdots W_1)^T, \quad p=1,\ldots,N.
\end{equation}
Here $E'(W)$ denotes the $d\times d$ matrix with entries
\begin{equation}
    \label{eq:big-grad3}
    E'(W)_{jk} = \frac{\partial E}{\partial W_{jk}}, \quad 1\leq j,k \leq d.
\end{equation}
This article focuses on the analysis of this gradient flow. 

\subsection{A comparison with deep learning}
\label{subsec:overparam}
Let $d=d_x=d_y$ and let $f_p: \R^d \to \R^d$ define the linear function $f_p(x)= W_px$ for $1 \leq p \leq n$. Then the linear function $f: \R^d \to \R^d$ corresponding to the matrix $f(x)=Wx$ is $f= f_N \circ f_{N-1} \ldots \circ f_1$. The output function $f$ depends only on the  end-to-end matrix $W$. However, the same function $f$ may be represented by $Nd^2$ choices of training parameters $\mathbf{W}$. Thus, despite the absence of the nonlinear activation element and shifts as in a neural network, the choice of variables in the DLN models {\em overparametrization\/}.

Natural learning tasks for the DLN, such as matrix completion, give rise to {\em degenerate loss functions\/}. Assume given a subset $S \subset \{(i,j)\}_{1\leq i,j \leq d}$ and assume given the values of $W_{ij}$, for $(i,j)\in S$, say $W_{ij}=a_{ij}$. The task in matrix completion task is to obtain a principled answer to the question: how do we reconstruct $W$ from the partial observations $a_{ij}$ for $(i,j)\in S$?

The DLN may be used to study this question as follows. We define the quadratic loss function
\begin{equation}
    \label{eq:mc1a}
    E(W) = \frac{1}{2}\sum_{(i,j)\in S} |W_{ij}-a_{ij}|^2,
\end{equation}
and seek the limit of $W(t)= \phi(\ww(t))$ as $t \to \infty$ when $\ww(t)$ solves the gradient flow~\eqref{eq:big-grad2}. 

This loss function is degenerate because it does not depend on the values $M_{ij}$ when the indices $(i,j)$ do not lie in $S$. Thus, the loss function is minimized on the affine subspace 
\[ \mathcal{S}=\{W \in \Md: W_{ij}=a_{ij}, \quad (i,j) \in S\}.\]

The analysis of the gradient flow~\eqref{eq:big-grad2} lies beyond the standard theory because of the interplay between overparametrization and the degeneracy of the loss function.
Here by standard theory, we mean the analysis of gradient flows using the main results on convergence of the gradient flow such as La Salle's invariance principle and the closely related Barbashin-Krasovskii criterion~\cite{Barbashin,LaSalle}, the use of the \Lojas\/ convergence criterion~\cite{Lojas,Simon} for analytic loss functions, and the Morse-Smale decomposition of the gradient flow~\cite{Milnor} when the loss function is Morse. These results provide the typical framework for the analysis of gradient flows, but they cannot be naively applied to the DLN. The use of La Salle's invariance principle is not valid when the loss function $E(W)$ does not grow sufficiently fast as $|W|\to \infty$ (for example, for matrix completion). The Morse-Smale decomposition does not hold since $\phi \circ E(\ww)$ is not a Morse function due to overparametrization.

\subsection{Main results}
\label{sec:results}
The theorems in this article are a composite of results obtained by several authors. They have been chosen to illustrate three geometric aspects.

\begin{enumerate}
    \item {\em Foliation by invariant varieties; the $\mathbf{G}$-balanced varieties\/}. 
    The phase space $\Md^N$ is foliated by invariant varieties described by quadratic matrix equations parametrized by $\mathbf{G} \in \symm_d^{N-1}$. 
    Theorem~\ref{thm:ACH} explains this structure. It is a geometric consequence of `thin gradients' and holds for all $E$. 
    \item {\em Riemannian gradient flows\/.} 
    The dynamics on the invariant varieties may be explicitly described  when $\mathbf{G}=\mathbf{0}$. The variety $\balance_{\mathbf{0}}$ is itself foliated by manifolds corresponding to $W$ with rank $r$, $1\leq r \leq d$. We refer to these as {\em balanced manifolds\/}, denoted $\balance_r$, or simply $\balance$, when $r=d$.  Each balanced manifold $\balance_r$ is invariant and the dynamics on the balanced manifolds is described by a Riemannian gradient flow, with a metric that may be computed exactly using Riemannian submersion. This metric has an obvious infinite-depth limit (though the Riemannian submersion itself does not). 
    \item {\em Group orbits, entropy and stochastic dynamics\/.} We define a Boltzmann entropy of the form $\log \volume ( \orbitx)$ for group orbits on $\balance_r$. This entropy may be computed explicitly (see Theorem~\ref{thm:entropy}). The microscopic dynamics associated to the entropy are described using a Riemannian Langevin Equation (RLE) which reveals the role of curvature in the dynamics. 
\end{enumerate}
The rigorous results are only part of the story. Each of the main theorems formalizes a different form of geometric intuition and was guided by heuristics, numerical experiments and connections with other areas of mathematics. These connections are truly surprising. For example, the balanced varieties have a subtle relation with the theory of minimal surfaces; a special case of the Riemannian geometry of the DLN is the Bures-Wasserstein geometry on $\psd_d$; and the entropy formula is obtained by analogy with Dyson Brownian motion in random matrix theory. 

For these reasons, we first state these results along with some background. This allows the reader to obtain
an overview of the geometry of the DLN without much baggage. Most proofs are presented in the sections that follow. 

\section{$\mathbf{G}$-balanced varieties are invariant}
\subsection{Definitions}
Denote the coordinates of $\mathbf{G} \in \symm_d^{N-1}$ by 
\begin{equation}
    \label{eq:G-def}
    \mathbf{G}=(G_{N-1},\cdots, G_1).
\end{equation}
Given $\mathbf{G}$, consider the system of $N-1$ quadratic equations
\begin{equation}
    \label{eq:b1}
    W_{p+1}^TW_{p+1} = W_pW_p^T -G_p , \quad 1 \leq p \leq N-1.
\end{equation}
The solution set defines an algebraic variety that is termed the $\mathbf{G}$-balanced variety. We denote it by
\begin{equation}
    \label{eq:b2}
    \balance_{\mathbf{G}} = \{ \mathbf{W}\in \Md^N \left| W_{p+1}^TW_{p+1} = W_pW_p^T  - G_p, \; 1 \leq p \leq N-1. \right. \} 
\end{equation}
Not all values of $\mathbf{G}$ give rise to non-empty $\balance_{\mathbf{G}}$. However, given a point $\mathbf{W}\in \Md^N$, we may use equation~\eqref{eq:b1} to define $\mathbf{G}$. Thus, the space $\Md^N$ is fibered by the varieties $\balance_\mathbf{G}$.

When $\mathbf{G}=\mathbf{0}$, equation~\eqref{eq:b1} reduces to the important special case
\begin{equation}
    \label{eq:b1-main}
    W_{p+1}^TW_{p+1} = W_pW_p^T, \quad 1 \leq p \leq N-1.
\end{equation}
These equations define the {\em balanced variety\/}
\begin{equation}
    \label{eq:b7}
    \balance_{\mathbf{0}} = \{ \mathbf{W}\in \Md^N \left|  W_{p+1}^TW_{p+1}= W_pW_p^T , \; 1 \leq p \leq N-1. \right. \} 
\end{equation}
If $\ww \in \balance_{\mathbf{0}}$, the singular values, and thus rank, of each $W_p$ are the same. In Section~\ref{sec:parametrization}, we use this observation to construct a parametrization which shows that $\balance_{\mathbf{0}}$ is foliated by manifolds $\balance_r$ corresponding to the rank $r$, $1\leq r \leq d$. We refer to the leaf of $\balance_{\mathbf{0}}$ with rank $r=d$ as the {\em balanced manifold\/} and denote it by $\balance$. Our main theorems concern the behavior on $\balance_r$.

We lack a complete understanding of the singularities of the variety $\balance_{\mathbf{G}}$. However, it is easy to check that $\balance$ is a manifold. For each $p$, equation~\eqref{eq:b1} is $\symm_d$-valued; thus, there are $(N-1) d(d+1)/2$ scalar equations.  Since we have $Nd^2$ parameters,  we find that 
\begin{equation}
    \label{eq:b3b}
    \mathrm{dim}(\balance)= d^2 + (N-1) \frac{d(d-1)}{2}.
\end{equation}
Of these, $d^2$ degrees of freedom correspond to an end-to-end matrix $W\in \Md$ and the remaining $(N-1)d(d-1)/2$ degrees of freedom correspond to an $\od^{N-1}$ group orbit. We parametrize $\balance_{\mathbf{G}}$ in Section~\ref{sec:parametrization} to make this explicit.  We focus mainly on $\balance$ for simplicity. When we consider the rank-deficient cases, $\balance_r$ with $r<d$, the $\od^{N-1}$ orbits have to be replaced by $\Stief_{r,d}^{N-1}$ orbits.

\subsection{Dynamics on invariant manifolds}
Each variety $\balance_{\mathbf{G}}$ is invariant under the dynamics. Let us state this assertion precisely. 

Assume that $E\in C^2$ and consider the initial value problem
\begin{equation}
    \label{eq:ivp}
    \dot{\mathbf{W}} = -\nabla_{\mathbf{W}} E\circ 
    \phi (W), \quad \mathbf{W}(0)=\mathbf{W}_0.
\end{equation}
This is a differential equation with a locally Lipschitz vector field. Thus, Picard's theorem guarantees the existence of a unique solution on a maximal time interval $(T_{\min},T_{\max})$ containing $t=0$. Let $\mathbf{G}$ be given by the initial condition
\begin{equation} 
\label{eq:newG}
G_p = W_{p+1}^T(0) W_{p+1}(0) - W_{p}(0) W_{p}^T(0), \quad 1\leq p \leq n. 
\end{equation}
\begin{theorem}[Arora, Cohen, Hazan~\cite{ACH}]
\label{thm:ACH}
The following hold on the maximal interval of existence of solutions to equation~\eqref{eq:ivp}.
\begin{enumerate}   
\item[(a)] The solution $\mathbf{W}(t)$ lies on the variety $\balance_{\mathbf{G}}$.
\item[(b)] The end-to-end matrix $W(t)$ satisfies 
\begin{equation}
    \label{eq:closed0}
    \dot{W} = - \sum_{p=1}^N (A_{p+1}A_{p+1}^T) \,E'(W) \,  (B_{p-1}^T B_{p-1}), 
\end{equation}
where $A_{N+1}=B_0=1$ and we have defined the partial products
\begin{equation}
    \label{eq:closed0b}
    A_p = W_N \cdots W_{p}, \quad B_p = W_p \cdots W_1, \quad 1 \leq p \leq N.
\end{equation}
\end{enumerate}
\end{theorem}
The specific nature of $E$ is irrelevant to
Theorem~\ref{thm:ACH}. Instead, it reflects a fundamental geometric restriction forced by overparametrization. The gradient vector fields $\nabla_{\ww}L(\ww)$ always lie along a `thin' space of dimension $d^2$ in $\Md^N$. We explain this point further in Remark~\ref{rem:thin-gradients} below. 
Theorem~\ref{thm:ACH} tells us that each variety $\balance_{\mathbf{G}}$ is invariant under the dynamics, but it does not provide a closed description of the reduced dynamics. However, on the balanced manifold $\balance$ we have
\begin{theorem}[Arora, Cohen, Hazan~\cite{ACH}]
    \label{thm:ACH2}
Assume $\mathbf{W}(0) \in \balance$. The end-to-end matrix $W(t) =\phi(\mathbf{W}(t))$ satisfies the differential equation
\begin{equation}
    \label{eq:closed1}
    \dot{W} = - \sum_{k=1}^N(WW^T)^{\tfrac{N-k}{N}}\, E'(W) \, (W^TW)^{\tfrac{k-1}{N}}.
\end{equation}
\end{theorem}
Theorem~\ref{thm:ACH} and Theorem~\ref{thm:ACH2} are easy to establish (see Section~\ref{sec:proofs}). However, despite their simplicity, both Theorems have a fascinating character. The most important structural feature is that equation~\eqref{eq:closed1} is a {\em Riemannian\/} gradient flow in disguise. We explain this idea below and then conclude with some remarks on  these theorems.

\subsection{Riemannian gradient descent}
We recommend~\cite{Lee-riemannian} as an introduction to Riemannian geometry. The reader willing to take some definitions on faith can also understand the main ideas in this work without a detailed understanding of Riemannian geometry.

A metric $g$ on $\Md$ assigns a length to each tangent vector $Z\in T\Md$. Since the tangent space to $\Md$ at any point is itself isomorphic to $\Md$, a metric is an assignment of lengths of the form
\begin{equation}
    \label{eq:metric1} g(W)(Z,Z) := \|Z\|_{g(W)}^2, \quad Z \in T_W\Md.
\end{equation}
The inner-product $g(W)(Z_1,Z_2)$ may be recovered from the polarization identity
\begin{equation}
    \label{eq:metric1a} g(W)(Z_1,Z_2) = \frac{1}{4}\left(\|Z_1+Z_2\|_{g(W)}^2- \|Z_1-Z_2\|_{g(W)}^2\right).
\end{equation}

We define a metric $g^N$ on $\Md$ as follows. Given the depth $N$ for every $W\in \Md$ we define the linear map
\begin{equation}
    \label{eq:metric2}
    \mathcal{A}_{N,W}: T_W \Md \to T_W \Md, \quad Z \mapsto \sum_{k=1}^N(WW^T)^{\tfrac{N-k}{N}}\, Z \,(W^TW)^{\tfrac{k-1}{N}}.
\end{equation}
This linear map allows us to rewrite equation~\eqref{eq:closed1} as
\begin{equation}
    \label{eq:closed2}
    \dot{W} = -\mathcal{A}_{N,W} \left(E'(W)\right).
\end{equation}
We will show that the operator $\mathcal{A}_{N,W}$ is always invertible when $W$ has full rank. In this setting, the structure of equation~\eqref{eq:closed2} may be understood better. Define 
\begin{equation}
    \label{eq:metric3} g^N(W)(Z,Z) = \Tr(Z^T \mathcal{A}_{N,W}^{-1} Z).
\end{equation}
It then follows from the polarization identity and the identity $(\mathcal{A}_{N,W}(Z))^T = \mathcal{A}_{N,W}(Z^T)$ that
\begin{equation}
    \label{eq:metric4} g^N(W)(Z_1,Z_2) = \Tr(Z_1^T \mathcal{A}_{N,W}^{-1} Z_2).
\end{equation}

Given a Riemannian manifold $(\mathcal{N},h)$ and a differentiable function $E:\mathcal{N}\to \R$, the  Riemannian gradient of $E$ is a vector in $T_x \mathcal{N}$ obtained from the duality pairing
\begin{equation}
    \label{eq:metric5}
    dE(x)(z) = h(\grad_h E(x),z), \quad z \in T_x\N.
\end{equation}
In our setting, $(\mathcal{N},h)= (\Md,g^N)$ where $g^N$ is given by equation~\eqref{eq:metric3} for $Z \in T_M\Md$. Then the left and right hand sides of this equation are 
\begin{equation}
    \label{eq:metric6}
    dE(W)(Z) = \Tr(Z^T E'(W)), \quad g^N(\grad_{g^N} E(W),Z) = \Tr(Z^T \mathcal{A}_{N,W}^{-1} E'(W)). 
\end{equation}

A more careful analysis~\cite{Bah2019} reveals that the assumption that $W$ have full rank is not necessary. Following~\cite{Bah2019} we foliate $\Md$ by rank, defining 
\begin{equation}
    \label{eq:rankfol1}
    \Mr = \{ W \in \Md \left| \mathrm{rank}(W) = r \right. \}.
\end{equation}
The balanced manifolds $\balance_r$ are naturally related to $\Mr$: $\ww \in \balance_0$ lies in $\balance_r$ if and only if $W=\phi(\ww) \in \Mr$. Further, it is shown in~\cite[\S 3]{Bah2019} that the metric $g^N$ restricts naturally to $\Mr$ for $r < d$. We then have
\begin{theorem}[Bah, Rauhut, Terstiege, Westdickenberg~\cite{Bah2019}]
    \label{thm:BRTW2}
Equation~\eqref{eq:closed1} is equivalent to the Riemannian gradient flow on $(\Mr,g^N)$
\begin{equation}
    \label{eq:closed3}
    \dot{W} = - \grad_{g^N} E(W).
\end{equation}
In particular, when $\ww(0)\in \balance_r$, the end-to-end matrix $W(t)=\phi(\ww(t))$ evolves in $\Mr$ according to this Riemannian gradient flow.
\end{theorem}
When $W(t)$ evolves according to~\eqref{eq:closed3} we have the fundamental estimate
\begin{equation}
    \label{eq:gradient-estimate}
    \frac{dE(W(t))}{dt} = -\|\mathrm{grad}_{g^N} E \|_{g^N}^2 \leq 0.
\end{equation}
It follows that when $W(t)$ converges to a critical point of $E$, then $\ww(t)$ converges to a critical point of $L=E\circ \phi$ at exactly the same rate. Overparametrization does not change the speed of convergence.

\subsection{Remarks on the invariant manifold theorems}
\begin{remark}
 \label{rem:ach1}
Equation~\eqref{eq:closed1} was first obtained by using the gradient flow~\eqref{eq:big-grad1} and  the identity $W=W_NW_{N-1}\cdots W_1$~\cite{ACH}. 
The underlying metric $g^N$ was then obtained by noting the properties of the map $\mathcal{A}_{N,W}$, including its restriction to rank $r$ in~\cite{Bah2019}. These calculations may be explained geometrically using Riemannian submersion as discussed in Section~\ref{sec:entropy}.
 \end{remark}
\begin{remark}[Thin gradients]
    \label{rem:thin-gradients}
    The space of gradients of loss functions of the form $L=E \circ \phi$ at a point $\ww \in \Md^N$ has at most $d^2$ dimensions.  To see this, fix $l,m$ and choose linear energies $E_{lm}(W):=W_{lm}$. There are $d^2$ such energies; thus, their gradients form a basis for the space of gradients of energies $E$  at $W \in \Md$. Equation~\eqref{eq:big-grad2} then shows that the space of gradients of $L=E\circ \phi$ at $\mathbf{W}$ has $d^2$ dimensions. However, the dimension of the space $\Md^N$ is $Nd^2$. 
\end{remark}
\begin{remark}[Group orbits]
    \label{rem:group-orbits}
Since the gradient flow can explore at most $d^2$ directions we have $(N-1)d^2$ free parameters. Of these, $(N-1)d(d+1)/2$ parameters are fixed by the constants $(G_N,\ldots,G_1)$. We use this observation to parametrize $\ww$ by $\Md \times O_d^{N-1}$ using the polar factorization in Section~\ref{sec:parametrization}. In particular, for each $W\in \Mr$ the inverse image $\phi^{-1}(W) \cap \balance$ is a $\Stief_{r,d}^{N-1}$ orbit $\orbitx$.
\end{remark}

\begin{remark}[Global existence and convergence]
The following method for establishing convergence was introduced in~\cite{Bah2019}.

Picard's theorem provides local existence on an interval $(T_{\min},T_{\max})$ with $T_{\min} < 0 < T_{\max}$. In order to obtain global existence for $\ww(t)$, it is only necessary to show that $W(t)=\phi(\ww(t))$ is bounded in time. If so, we have the uniform bounds 
\begin{equation}
    \|W_{p+1}\|^2 = \Tr(W_{p+1}^TW_{p+1}) = \Tr(W_{p}W_{p}^T - G_p) = \|W_p\|^2 - \Tr(G_p).
\end{equation}
Thus, we may inductively control $\|W_N(t)\|$,$\ldots$,$\|W_2(t)\|$ in terms of $\|W_1(t)\|$. Since $W= W_N \cdots W_1$, the bound on $\|W_1(t)\|$ is equivalent to a bound on $\|W\|$. It then follows from a continuation argument that $T_{\min}=-\infty$ and $T_{\max}=+\infty$ if $\|W(t)\|$ is uniformly bounded on its interval of existence.

Once we have established that $\ww(t)$ remains in a compact set, we may use standard criterion to establish convergence as $t\to \infty$. When $E$ is analytic, the \Lojas\/ criterion implies convergence to a critical point. For $E \in C^2$, we may use the La Salle invariance principle to classify the $\omega$-limit set $\omega(\ww(0))$.
\end{remark}
 \begin{remark}[Relation to the Simons cone]
 The variety $\balance_{\mathbf{G}}$ is a conic section in $\Md^N$ with a parametrization discussed in Section~\ref{sec:parametrization}. We may understand  the balanced manifold $\balance$ explicitly when $d=2$ and $N=2$. Let 
 \begin{equation}
     \label{eq:simons1}
     W_1 = \left(\begin{array}{cc}
          x_1 &x_2  \\
          x_3 & x_4 
     \end{array}\right), \quad    
     W_2 = \left(\begin{array}{cc}
          x_5 &x_6  \\
          x_7 & x_8 
     \end{array}\right). 
 \end{equation}
 Then equation~\eqref{eq:b1-main} reduces to a system of three quadratic equations
 \begin{eqnarray}
     \label{eq:simons2}
     x_5^2+x_7^2 &=& x_1^2+x_2^2, \\
     \label{eq:simons3} x_5x_6+ x_7x_8 & = & x_1x_3+x_2 x_4 \\
     \label{eq:simons4} x_6^2+x_8^2 & =&  x_3^2 + x_4^2.
 \end{eqnarray}
When we add equations~\eqref{eq:simons2} and ~\eqref{eq:simons4} we obtain the {\em Simons cone}
\begin{equation}
    \label{eq:simons5}
\mathcal{C} =\{x \in \R^8 \left| x_1^2 +x_2^2 +x_3^2 +x_4^2  = x_5^2 +x_6^2 +x_7^2 +x_8^2\right. \} \subset \R^8.
\end{equation}
The Simons cone is a fundamental counterexample in the calculus of variations and geometric measure theory. The variety $\mathcal{C}$ defined by~\eqref{eq:simons5} has zero mean curvature at all points $x \neq 0$ and is a local minimizer of the area functional. However, it is not smooth because of the singularity at $x=0$. On the other hand, there are no such singularities for minimal surfaces in $\R^n$ when $n \leq 7$. Thus, the Simons cone is the simplest stable minimal surface with a singularity, and it manifests only in $\R^n$ with $n\geq 8$.

We see that the balanced manifold $\balance$ is  a five-dimensional variety contained within the Simons cone $\mathcal{C}$. We will use a stochastic extension of equation~\eqref{eq:big-grad1} to shed some light on this unexpected connection.
 \end{remark}

 \begin{remark}
     \label{rem:bures}
     The space of positive definite matrices $\psd_d$ may be equipped with the Bures-Wasserstein metric $\gbw$~(see~\cite{Bhatia-Jain-Lim} for an exposition). Gradient flows on $(\psd_d,\gbw)$ can be obtained from the DLN as follows. We set $N=2$ and further restrict $\ww$ to the subspace 
     \begin{equation}
         \label{eq:gbw}
         \mathcal{V}= \{ \ww \in \Md^2 \left| W_2=W_1^T \right.\}.
     \end{equation}
Then $W=\phi(\ww) = W_1^TW_1$ is a psd matrix and the metric $g^2$ given by Riemannian submersion of $\mathcal{V}$ is the Bures-Wasserstein metric $\gbw$.
 \end{remark}

\section{Parametrization of $\balance_{\mathbf{G}}$ and $\balance$} 
\label{sec:parametrization}
\subsection{Polar factorization and the SVD}
Recall that a matrix $W \in \Md$ admits left and right polar decompositions
\begin{equation}
    \label{eq:polar1}
    W = Q P, \quad W= R U^T
\end{equation}
where $P, R \in \psd_d$ and $Q, U\in \od$. Further, $P = \sqrt{W^TW}$ and $R=\sqrt{WW^T}$ are the unique psd square roots of these matrices.

The polar factorization is related to the singular value decomposition (SVD) of $W$ as follows. Let us denote the SVD by
\begin{equation}
    \label{eq:svd}
    W = V_R \Sigma V_P^T,  
\end{equation}
where $V_R,V_P \in O_d$ and $\Sigma = \diag(\sigma_1, \ldots,\sigma_d)$ with each $\sigma_i\geq 0$.  
Then we have immediately
\begin{equation}
    \label{eq:svd2}
 P= \sqrt{W^TW} = V_P \Sigma V_P^T, \quad   R = \sqrt{WW^T} = V_R \Sigma V_R^T. 
\end{equation}
Therefore, $V_P$ and $V_R$ provide an orthonormal basis of eigenvectors for $P$ and $R$ respectively. Finally, 
we have the relation between the $O_d$ factors
\begin{equation}
    \label{eq:svd3}
 Q= V_R V_P^T, \quad   U = V_P V_R^T. 
\end{equation}

\subsection{Parametrization by the polar factorization}
\label{sec:polar-parametrization}
 We now parametrize each $\mathbf{G}$-balanced variety $\balance_\mathbf{G}$ using the polar factorization, defining a map
\begin{equation}
    \label{eq:b4}
    \manmap: O_d^{N-1}  \times \Md \to \Md^N, \quad (Q_{N},\ldots Q_{2},W_1) \mapsto (W_N,\ldots, W_{1}).
\end{equation}

An interesting feature of equation~\eqref{eq:b1} is the manner in which the left and right polar factors alternate along the network as the index $p$ varies. 
We use the polar factors of $W_p$ to rewrite equation
~\eqref{eq:b1} in the form
\begin{equation}
    \label{eq:b1b}
    P_{p+1}^2 = R_{p}^2 -G_p, \quad 1 \leq p \leq N-1.
\end{equation}
 We define the map $\mathfrak{x}$ as follows.
 Compute
 \begin{equation}
R_1^2=W_1W_1^T,\quad  P_2 = \sqrt{R_1^2 - G_1}, \quad W_{2}=Q_2 P_2.   
\end{equation}
Now proceed algorithmically: as $p$ increases, we find sequentially
\begin{equation}
\label{eq:b6}
 R_p^2 = W_pW_p^T, \quad P_{p+1} = \sqrt{R_{p}^2 - G_p}, \quad W_{p+1}=Q_{p+1} P_{p+1}.   
\end{equation}
These equations may also be combined into an iterative sequence for $P_{p}$. We have
\begin{equation}
\label{eq:b6b}
P_1=W_1^T W_1,\quad P_{p+1} = \sqrt{Q_pP_{p}^2Q_p^T - G_p}. 
\end{equation}

Each square-root is a smooth map when $G_p$ is a negative definite matrix. Thus, the variety $\balance_{\mathbf{G}}$ is a manifold with dimension $d^2+ (N-1)d(d-1)/2$ when each $G_p$ is negative definite.

The balanced varieties are matrix-valued conic sections. A useful geometric caricature is obtained by considering the case $d=1$.~\footnote{We switch notation to lower case letters to prevent any confusion with the case $d \geq 2$.} Assume given $( v_{N-1},\ldots, v_1, w_1)$ where each $v_p = \pm 1$. Then given $\mathbf{g}:=(g_{N-1},\ldots,g_1)$, we see that $\balance_{\mathbf{g}}$ is a conic section in $\R^N$ described by the hyperbolas
\begin{equation}
\label{eq:b6q}
 w_{p+1} = v_{p+1}\sqrt{w_{p}^2 - g_p}, \quad 1 \leq p \leq N-1.
\end{equation} 
Finally, note that the parametrization goes from right to left. We have an analogous parametrization from left to right
\begin{equation}
    \label{eq:lr-param}
    \manmapbeta:  \Md \times O_d^{N-1}  \to \Md^N, \quad (W_N, U_{N-1},\ldots U_{1}) \mapsto (W_N,\ldots, W_{1}),
\end{equation}
given for $N-1\geq p \geq 1$ by the sequential polar factorizations
\begin{equation}
\label{eq:b6s}
 P_{p+1}^2 = W_{p+1}^TW_{p+1}, \quad R_{p} = \sqrt{P_{p+1}^2 +G_p}, \quad W_{p}= R_{p} U_p^T.   
\end{equation}

\subsection{Parametrization by the SVD}
The parametrization may be simplified further on the balanced variety $\balance_{\mathbf{0}}$. We now have
\begin{equation}
    \label{eq:new-b1}
    W_{p+1}^T W_{p+1} = W_p W_p^T
\end{equation}

Let us first study the case when $W_N$ has full rank. Then $P_N$ is positive definite and the parametrization $\mathfrak{x}$ continues to be smooth because equation~\eqref{eq:b6} reduces to the chain of equalities 
\begin{equation}
    \label{eq:b8} R_{p} = P_{p+1}, \quad W_p = R_p V_p^T, \quad P_p =\sqrt{W_p^T W_p} = V_p R_p V_p^T.  
\end{equation}
It follows that each $P_p$ is positive definite and that 
$\balance_{\mathbf{0}}$ is locally a manifold. We denote this branch of $\balance_{\mathbf{0}}$ as $\balance$ and call it the {\em balanced manifold}.
  
The singular values of each $W_p$ are identical on $\mathcal{M}$. This allows us to simplify the  parametrization $\mathfrak{x}$ further. We define the coordinate map 
\begin{equation}
    \label{eq:b10} \mathfrak{z}: \R_+^d \times O_d^{N+1} \to \balance, \quad (\Lambda, Q_N, \ldots, Q_0) \to (W_N,\ldots,W_1),
\end{equation}
where for each $1 \leq p \leq N$ we set
\begin{equation}
\label{eq:b11}
W_p = Q_p \Lambda Q_{p-1}^T,  \quad\mathrm{and}\quad \Lambda = \Sigma^{1/N}.
\end{equation}
Here $\Lambda$ denotes the singular values of each $W_p$ and $\Sigma$ denotes the singular values of $W=W_N \cdots W_1$. Indeed, it follows from equation~\eqref{eq:b11} that
\begin{equation}
\label{eq:b12}
W = Q_N \Sigma Q_0^T.    
\end{equation} 
This parametrization is a local bijection when the singular values are distinct. However, it can fail to be smooth when $\Sigma$ has repeated singular values. We will compute the metric on $\balance$ in Section~\ref{sec:embed}. 

\subsection{The rank-deficient case}
\label{subsec:rank-def}
When $W$ has rank $r <d$, we may extend each of the above parametrizations in a natural way. 
First, we fix the action of the permutation group on $\Sigma$ by ordering the singular values
\begin{equation}
    \label{eq:svd4}
    \sigma_1 \geq \sigma_2 \ldots \geq \sigma_d.
\end{equation}
Thus, when $W$ has rank $r <d$, 
\begin{equation}
    \label{eq:svd5}
    \sigma_{d-r+1}=\ldots = \sigma_d =0.
\end{equation}
Taking a limit of full-rank matrices in equation~\eqref{eq:b11}, we see that it is only the first $r$ orthonormal vectors in each $Q_p$ that affect $\ww$ in the rank-deficient limit. Thus, in order to extend $\mathfrak{z}$ to $\balance_r$ we must replace the action of $\od$ in the parametrization with the action of $\Stief_{r,d}$. Then~\eqref{eq:b10} extends to the family of parametrizations indexed by $r$, $1\leq r \leq d$
\begin{equation}
    \label{eq:b14} \mathfrak{z}_r: \R_+^d \times \Stief_{r,d}^{N+1} \to \balance_r, \quad (\Lambda, Q_N, \ldots, Q_0) \to (W_N,\ldots,W_1),
\end{equation}
Equations~\eqref{eq:b11}--\eqref{eq:b12} continue to hold true and we see immediately that $\balance_r$ is also a manifold. The image $\phi(\balance_r)=\Mr$, where $\Mr$ was defined in equation~\eqref{eq:rankfol1}.

We define the parametrizations $\mathfrak{x}_r$ and $\mathfrak{y}_r$ by extending ~\eqref{eq:lr-param} and~\eqref{eq:b14} in an analogous manner to $r< d$.

\section{Entropy of group orbits and Riemannian submersion}
\label{sec:entropy}
\subsection{Overview}
In this section, we use Riemannian submersion to develop a thermodynamic framework for the DLN. We restrict ourselves to full rank matrices. The generalization to rank $r$ is natural, but it requires that $\od$ be replaced by $\Stief_{r,d}$ and a careful treatment of zero singular values.

\subsection{The entropy formula}
The parametrization~\eqref{eq:b10}--\eqref{eq:b12} allows us to foliate $\balance$ with group orbits as follows. For each $W \in \Mdd$ consider its preimage in $\balance$
\begin{equation}
    \label{eq:ent1}
    \orbitx = \phi^{-1}(W)=\{ \ww \in \balance \left| W_N \cdots W_1 = W \right. \}.
\end{equation}
We use equations~\eqref{eq:b10}--\eqref{eq:b12} to find that $\orbitx$ is an orbit of $O_d^{N-1}$. Indeed, the SVD of $W$ fixes $Q_0$, $\Sigma$ and $Q_N$, leaving $(Q_{N-1},\ldots, Q_1)$ free. 

We interpret this geometric picture in the language of thermodynamics. Matrices $W \in \Mdd$ downstairs are macrostates, or {\em observables\/}, whereas matrices $\ww \in \balance$ upstairs are microstates. Conceptually, the entropy enumerates the number of microstates corresponding to a given macrostate, with the enumeration given by Boltzmann's formula, $S=\log (\#)$, where $(\#)$ denote the number of microstates. In our setting, since $\balance$ inherits a metric from its embedding in $\Md^N$, the number of microstates associated to $W \in \Mdd$ is the volume of the group orbit $\orbitx$ with respect to this metric. This reduces the computation of the entropy to an evaluation of a matrix integral and we find
\begin{theorem}[Menon, Yu~\cite{MY-dln}]\label{thm:entropy}
    For each $W\in (\Mdd,g^N)$ with SVD given by~\eqref{eq:b12}
    \begin{align}
        \mathrm{vol}(\orbitx)=c_d^{N-1}\sqrt{\frac{\mathrm{van}(\Sigma^2)}{\mathrm{van}(\Sigma^{\frac 2N})}}=c_d^{N-1}\prod_{1\leq i<j\leq d}\sqrt{\frac{\sigma_i^{2}-\sigma_j^{2}}{\sigma_i^{\frac2N}-\sigma_j^{\frac2N}}}
    \end{align}
    where $c_d$ is the volume of the orthogonal group $\od$ with the standard Haar measure. 
    \end{theorem}
Here $\mathrm{van}(A)$ denoted the Vandermonde determinant associated to the matrix $A=\diag(a_1,\ldots,a_d)$ as follows:
\begin{equation}
    \label{eq:van}
    \mathrm{van}(A) = \left|\begin{array}{llll}
         1 & 1 &  & 1   \\
         a_1 & a_2 & & a_d \\
         a_1^2 & a_2^2 & \cdots & a_d^2  \\
         \vdots &  & & \vdots\\ 
         a_1^{d-1}& a_2^{d-1}& \cdots  &a_d^{d-1}
    \end{array}
    \right| = \prod_{1 \leq j < k \leq d} (a_k-a_j).
\end{equation}
Similar factors appear in a different way in the following
\begin{theorem}[Cohen, Menon, Veraszto~\cite{CMV}]
\label{thm:cmv}
 The volume form on $(\Mdd,g^N)$ is 
 \begin{equation}
     \label{eq:volume-form}
     \sqrt{\det g^N}\, dW = \det(\Sigma^2)^{\frac{N-1}{2N}}\mathrm{van}(\Sigma^{\frac{2}{N}}) \, d\Sigma \,dQ_0\, dQ_N.
 \end{equation}
\end{theorem}
Both theorems follow from explicit computations of metrics, but the two theorems  correspond to different metrics. Theorem~\ref{thm:cmv} concerns the metric downstairs and follows easily from Lemma~\ref{le:diagon}. On the other hand, Theorem~\ref{thm:entropy} concerns the metric upstairs $(\balance,\iota)$ as explained in Section~\ref{sec:embed}. It follows from a computation of the pullback metric $\mathfrak{z}^\sharp\iota$ on the parameter space $\R_+^d \times \od^{N+1}$.

\subsection{Equilibrium thermodynamics}
\label{subsec:entropy}
\begin{defn}
\label{def:entropy}  
Given $W \in (\Mdd,g^N)$ we define the Boltzmann entropy
    \begin{equation}
        \label{eq:ent2}
        S(W) = \log \mathrm{vol}(\orbitx).
    \end{equation}
At the inverse temperature $\beta \in (0,\infty)$ we define the {\em free energy\/}
\begin{equation}
    \label{eq:free-energy}
    F_\beta(W) = E(W) - \frac{1}{\beta} S(W).
\end{equation}
\end{defn}

The introduction of the entropy allows us to extend gradient descent of the energy $E(W)$ to gradient descent of the free energy  on $(\Mdd,g^N)$ 
\begin{equation}
    \label{eq:free-energy2}
    \dot{W} = -\mathrm{grad}_{g^N} F_\beta (W).
\end{equation}
The geometry of group orbits also allows us to introduce natural microscopic dynamics that capture the notion of `fluctuations in the gauge'. In Section~\ref{sec:RLE} we extend equation~\eqref{eq:free-energy2} to Riemannian Langevin Equations of the form
\begin{equation}
    \label{eq:free-energy2b}
    dW_t = -\mathrm{grad}_{g^N} F_\beta (W_t)\, dt + dB^{\beta,g^N}_t,
\end{equation}
where $B_t^{\beta,g^N}$ is Brownian motion at inverse temperature $\beta$ on $(\Mdd, g^N)$.

The inclusion of the entropy provides a selection principle when $E(W)$ is degenerate. We expect that a typical solution to equation~\eqref{eq:free-energy2} is attracted to the minimizing set
\begin{equation}
    \label{eq:free-energy3}
    \mathcal{S}_\beta= \mathrm{argmin}_{W\in \Mdd} F_\beta(W).
\end{equation}
At present, we do not understand the structure of this set completely even for the simple energy corresponding to matrix completion (see Section~\ref{sec:open}). In particular, it is natural to conjecture that $\mathcal{S}_\beta$ consists of a single point $A_\beta \in \Mdd$ and that $\lim_{\beta \to \infty} A_\beta$ is selected by the training dynamics due to noise introduced through round-off errors. However, we do not have rigorous results in this direction. 

\subsection{Riemannian submersion}
The metric $g^N$ was first introduced in~\cite[\S 3]{Bah2019} using an analysis of $\Mr$. The parametrization~\eqref{eq:b10}-\eqref{eq:b11} provides a simple geometric explanation for the origin of this metric.

\begin{theorem}[Menon, Yu~\cite{MY-dln}]
\label{thm:submer}
For each rank $r$, $1\leq r\leq d$, the metric $g^N$ on $\Mr$ is obtained from the map $\phi: \balance_r \to \Mr$ by Riemannian submersion. 
\end{theorem}
The main ideas in the proof of this theorem are as follows: 
\begin{enumerate}
    \item Each manifold $\balance_r$ is Riemannian with the metric $\iota$ induced by its embedding in $\Md^N$. We present the essential ideas in the computation of this metric in Theorem~\ref{thm:embed} in  Section~\ref{sec:embed}.
    \item The $O_d^{N-1}$ action discussed in Section~\ref{sec:parametrization} is an isometry of $(\balance,\iota)$. This idea is implicit in the entropy formula and is made explicit in Theorem~\ref{thm:embed}.
    \item The differential $\phi_*:T\balance_r\to T\Mr$, its kernel and orthogonal complement may be computed explicitly using Theorem~\ref{thm:embed}. We then find an orthonormal basis of $(\mathrm{Ker}\,\phi_*)^{\perp}$ and show that the projection of this basis to $T\Mr$ provides an orthonormal basis for $(T_W\Mr,g^N)$.
\end{enumerate}

\section{Proof of Theorem~\ref{thm:ACH}--Theorem~\ref{thm:BRTW2}}
\label{sec:proofs}
In this section, we provide the essential steps in the proofs of Theorem~\ref{thm:ACH}--Theorem~\ref{thm:BRTW2}. The proofs involve direct matrix calculations. We typically assume the Einstein convention and sum over repeated indices.

\subsection{Balanced varieties}
\subsubsection{Overparametrization and equation~\eqref{eq:big-grad2}}
\label{subsec:chain-rule}
Fix $p$, $1\leq p \leq N$. We must compute the matrix $\nabla_{W_p} E(W)$ with $W=\phi(\ww)$. By the chain rule, this is the matrix with $(j,k)$ entries
\begin{equation}
    \label{eq:p1}
    \frac{\partial E(W)}{\partial W_{p,jk}} = \frac{\partial E(W)}{\partial W_{lm}} \frac{\partial W_{lm}}{\partial W_{p,jk}} := E'(W)_{lm} \frac{\partial W_{lm}}{\partial W_{p,jk}}
\end{equation}
Clearly, the form of $E$ is not that important, what really matters is the application of the chain rule to the product $W=W_NW_{N_1}\cdots W_1$.  
Let us write this as
\begin{equation}
    \label{eq:p2}
    W_{lm} = W_{N,li_{N-1}} W_{N-1,i_{N-1} i_{N-2}} \cdots W_{p,i_p i_{p-1}} \cdots W_{1,i_1 m}.
\end{equation}
Therefore,
\begin{equation}
    \label{eq:p3}
    \frac{\partial W_{lm}}{\partial W_{p,jk}} = W_{N,li_{N-1}} W_{N-1,i_{N-1} i_{N-2}} \cdots \delta_{i_p j}\delta_{i_{p-1} k} \cdots W_{1,i_1 m},
\end{equation}
which may be rewritten as
\begin{equation}
    \label{eq:p3b}
    \frac{\partial W_{lm}}{\partial W_{p,jk}} =  (W_{N}\cdots W_{p+1})_{lj} (W_{p-1}\cdots W_1)_{km}.
\end{equation}
It then follows from equation~\eqref{eq:p1} and~\eqref{eq:p3b} that 
\begin{equation}
\label{eq:p4}
\dot{W}_p = - \nabla_{W_p} E(W) = -(W_N \cdots W_{p+1})^T E'(W) (W_{p-1} \cdots W_{1})^T.    
\end{equation} 
\subsubsection{Proof of Theorem~\ref{thm:ACH}}
By the product rule we have
\begin{equation}
    \label{eq:p5}
    \frac{d}{dt} W_{p}W_{p}^T = \dot{W}_{p} W_{p}^T + W_{p} \dot{W}_{p}^T. 
\end{equation}
We now find from equation~\eqref{eq:p4} that
\begin{equation}
    \label{eq:p6}
    \dot{W}_{p} W_{p}^T = -(W_N \cdots W_{p+1})^T E'(W) (W_p W_{p-1} \cdots W_{1})^T,
\end{equation}
where we have observed that
\[ (W_{p-1} \cdots W_1)^T W_p^T = (W_p W_{p-1} \cdots W_{1})^T.\]
In a similar manner, using equation~\eqref{eq:p4} with $p$ replaced by $p+1$ we also have
\begin{equation}
    \label{eq:p6b}
    W_{p+1}^T \dot{W}_{p+1} = -(W_N \cdots W_{p+1})^T E'(W) (W_p W_{p-1} \cdots W_{1})^T.
\end{equation}
Part (a) of Theorem~\ref{thm:ACH} now follows immediately from the identity.
\begin{equation}
\label{eq:p8a}
\dot{W}_{p} W_{p}^T = W_{p+1}^T \dot{W}_{p+1}. 
\end{equation}

Part (b) of Theorem~\ref{thm:ACH} is similar. The factors $A_p$ and $B_p$ defined in equation~\eqref{eq:closed0b} appear naturally in the calculation. We use the chain rule to see that
\begin{equation}
\label{eq:p8}
\dot{W} = \dot{W}_N W_{N-1}\cdots W_1 + W_N \dot{W}_{N-1} W_{N-2}\cdots W_1 + W_N \cdots W_2 \dot{W}_1.
\end{equation}
We then apply equation~\eqref{eq:p4} to each term, observing that in each case the product above has factors that complement the products $(W_N \cdots W_{p+1})^T$ and $(W_{p-1} \cdots W_{1})^T$. For example,
the first term simplifies to
\[ \dot{W}_N W_{N-1}\cdots W_1 = -E'(W) (W_{N-1}\cdots W_1)^T W_{N-1}\cdots W_1 = -E'(W)B_{N-1}^TB_{N-1}.\]
The general term is given by
\[ W_N\cdots\dot{W}_p\cdots W_1  = -A_{p+1}A_{p+1}^T E'(W) B_{p-1}^T B_{p-1}. \]
We sum over the depth index $p$ from $1$ to $N$ to obtain equation~\eqref{eq:closed0}.

\subsubsection{Proof of Theorem~\ref{thm:ACH2}}
Theorem~\ref{thm:ACH2} is a specialization of Theorem~\ref{thm:ACH}(b) to the balanced variety. Since 
\[ W_{p+1}^T W_{p+1} = W_p W_p^T\]
on the balanced variety we may simplify the prefactors $A_{p+1}A_{p+1}^T$ and $B_{p-1}B_{p-1}^T$ appearing in equation~\eqref{eq:closed0}. It follows immediately from equation~\eqref{eq:closed0} and the parametrization~\eqref{eq:b11} that
\begin{equation}
    \label{eq:p9}
    A_p = W_N \cdots W_p = Q_N \Sigma^{\frac{N-p}{n}} Q_{p-1}^T, \quad B_p = W_p \cdots W_1 = Q_p  \Sigma^{\frac{p}{n}} Q_{0}^T.
\end{equation}
We then have
\begin{equation}
    \nonumber
    \label{eq:p11}
    A_pA_p^T = Q_N \Sigma^{\frac{2(N-p)}{n}} Q_N^T = (WW^T)^{\frac{N-p}{n}}, \quad B_p^T B_p = Q_0 \Sigma^{\frac{2p}{n}} Q_{0}^T = (W^TW)^{\frac{p}{n}}.
\end{equation}
We substitute in equation~\eqref{eq:closed0} to complete the proof.

\subsection{The Riemannian manifold $(\balance,g^N)$ and Theorem~\ref{thm:BRTW2}}
\label{subsec:metric}

We hold $N$ and $W$ fixed in this subsection. To simplify notation, we refer to $\mathcal{A}_{N,W}$ as $\aaa$. 

The main observation in this section is that $\aaa$ is diagonalized in singular value coordinates. Precisely, let $\{u_1,\cdots, u_d\}$ and $\{v_1,\cdots, v_d\}$ be the column vectors of $Q_N$ and $Q_0$ respectively. Then we may write the SVD~\eqref{eq:b12} as
\begin{equation}
    \label{eq:diagon1}
    W = Q_N \Sigma Q_0^T = \sum_{j=1}^d \sigma_j u_j v_j^T. 
\end{equation} 
\begin{lemma}
    \label{le:diagon}
    The operator $\aaa: T_W \Mdd \to T_W \Mdd$ is symmetric and positive definite with respect to the Frobenius inner-product. It has the spectral decomposition
    \begin{equation}
        \label{eq:diagon2}
        \aaa \, u_k v_l^T = \frac{\sigma_k^2-\sigma_l^2}{\sigma_k^{\frac{2}{N}}- \sigma_l^{\frac{2}{N}}} \, u_k v_l^T , \quad 1\leq k, l \leq d,
    \end{equation}
    when $k \neq l$ and 
\begin{equation}
        \label{eq:diagon2b}
        \aaa \, u_k v_k^T = N \sigma_k^{2-\tfrac{2}{N}}\, u_k v_k^T , \quad 1\leq k \leq d.
    \end{equation}
    \end{lemma}
    \begin{proof}
We have 
\[ WW^T = \sum_{i=1}^d \sigma_i^2 u_i u_i^T, \quad W^TW = \sum_{j=1}^d \sigma_j^2 v_j v_j^T.\]
We then find from equation~\eqref{eq:metric2} that for each pair $k$,$l$ 
\[ \aaa \, u_k v_l^T = \sum_{p=1}^N (WW^T)^{\frac{N-p}{N}} \, u_k v_l^T \,(W^TW)^{\frac{p-1}{N}} = \left( \sum_{p=1}^N \sigma_k^{\frac{2(N-p)}{N}} \sigma_l^{\frac{2(p-1)}{N}}\right) \, u_k v_l^T, \]
yielding~\eqref{eq:diagon2} and~\eqref{eq:diagon2b}.

Observe that distinct eigenvectors are orthogonal with respect to the Frobenius metric. Indeed, for each pair of indices $\{k,l\}$ and $\{m,n\}$ 
    \[ \Tr ((u_k v_l^T)^T u_m v_n) = \delta_{km}\delta_{ln},\]
which vanishes unless the pairs agree. It then follows from the eigendecomposition that $\aaa$ is a symmetric and positive definite operator from $T_W \Mdd \to T_W\Mdd$.
\end{proof}
We may represent an arbitrary matrix $Z \in T_W\Mdd$ in this basis as a sum $Z= Z_{kl} u_k v_l^T$. Then we use equation~\eqref{eq:metric3}  and Lemma~\ref{le:diagon} to obtain  
\begin{equation}
    \label{eq:p-metric1}
       g^N(Z,Z) = N \sum_{1 \leq k \leq d} \sigma_k^{2(1-\tfrac{1}{N})} Z_{kk}^2 +     
       \sum_{1\leq k, l \leq d, k \neq l} 
       \frac{\sigma_k^{\frac{2}{N}}- \sigma_l^{\frac{2}{N}}}{\sigma_k^2-\sigma_l^2}
       Z_{kl}^2.
\end{equation}
Theorem~\ref{thm:BRTW2} follows directly from the explicit diagonalization of the metric. 
\section{Embedding and the metric on $\balance$}
\label{sec:embed}
The balanced manifold $\balance$ is a Riemannian manifold since it is locally embedded in $\Md^N$ and inherits the Frobenius metric, denoted $\Fr$, from $\Md^N$. In order to use our parametrization $\mathfrak{z}$ to understand the Riemannian manifold $(\balance,\Fr)$, we must pull back the metric $\Fr$ onto the parameter space $\R_+^d \times \od^{N-1}$. In this section, we compute an orthonormal basis for $(T_\ww \balance,\Fr)$ by computing an orthonormal basis for the pullback metric $\mathfrak{z}^\sharp \Fr$. This is the technical core of~\cite{MY-dln}.

\subsection{The tangent space $T_{\ww}\balance$.}
First, at each point $\mathbf{W}\in \balance \subset \Md^N$, we compute the tangent space $T_{\ww}\balance$ by differentiating the parametrization~\eqref{eq:b11} as follows. The tangent space at the identity to the orthogonal group $\od$ is the space of antisymmetric matrices, denoted $\anti_d$. Thus, if the singular values $\Sigma$ of $W$ are distinct, the parametrization is smooth and for each $(\theta,\mathbf{A}) \in \R^d \times \anti_d^{N+1}$ we may define a smooth curve in $\ww(t) \in \balance$ using 
\[ \Lambda(t) = \Lambda+ t \theta, \quad Q_p(t) = e^{tA_{p}} Q_p,\quad \mathbf{W}(t)=\mathfrak{z}(\Lambda(t),\mathbf{Q}(t)),\]
where $\theta$ is the diagonal matrix $\diag(\theta_1, \ldots, \theta_d)$ and $\Lambda(t)=\Sigma(t)^{1/N}$.

Then we obtain a tangent vector in $T_\ww\balance$ by differentiating in $t$~\footnote{Note that the term `vector' here is used to mean `vector in the sense of vector space'. Each vector in $T_\ww \balance$ `sits in' $\Md^N$.}
\begin{equation}
    \label{eq:embed1}
 \left.\frac{d \ww(t)}{dt}\right|_{t=0} = D\mathfrak{z}(\ww)\left(\theta, \mathbf{A}\right).\end{equation}
Explicitly,  the $p$-th matrix in $D\mathfrak{z}(\ww)\left(\theta, \mathbf{A}\right)$ is 
\begin{equation}
\label{eq:basis1}
 D\mathfrak{z}(\ww)\left(\theta, \mathbf{A}\right)_p = A_p W_p + Q_p\theta Q_{p-1}^T - W_pA_{p-1}, \quad 1\leq p \leq N.
 \end{equation}

\subsection{Computing $\mathfrak{z}^\sharp \Fr$ in the standard basis.}
Since the metric on $\balance$ is inherited from its embedding in $\Md^N$, the length of each vector 
\begin{equation}
    \label{eq:embed2} \mathbf{V}:=D\mathfrak{z}(\ww)\left(\theta, \mathbf{A}\right) \in T_{\mathbf{W}}\balance
\end{equation}
is given by the Frobenius norm
\begin{equation}
    \label{eq:metric7} \|\mathbf{V}\|^2 = \sum_{k=1}^N \Tr(V_p^T V_p), \quad \mathbf{V}= (V_N, \ldots, V_1).
\end{equation}
Similarly, the inner product between two vectors $\mathbf{V}^{(i)} \in T_{\mathbf{W}}\balance$, $i=1,2$ is given by 
\begin{equation}
    \label{eq:metric7b} \langle \mathbf{V}^{(1)},
    \mathbf{V}^{(2)}\rangle = 
    \sum_{k=1}^N \Tr\left((V^{(1)}_p)^T,V^{(2)}_p\right).
\end{equation}

By linearity, we may reduce the computation of the pullback metric to  inner products between the action of $D\mathfrak{z}$ on the basis vectors on $\R^d \times \anti_d^{N+1}$. We arrange these basis vectors as follows.
First, let $\{e_i\}_{i=1}^d$ be the standard basis on $\R^d$ and define the image of diagonal matrices
\begin{equation}
    \label{eq:new-basis1}
    \mathbf{l}_i = D\mathfrak{z}(\ww)\left(e_i, \mathbf{0}\right), \quad i =1, \ldots, d.
\end{equation}
Next, form the standard basis for $\anti_d$
\begin{equation}
\label{eq:basis4}
\alpha_{ij} = \frac{1}{\sqrt{2}}(e_i e_j^T - e_j e_i^T), \quad 1 \leq i < j \leq d,
\end{equation}
and for each $0\leq p \leq N$ use it to define a tangent vector in $T_\mathbf{W}\balance$ by setting
\begin{equation}
    \label{eq:new-basis1a}
\mathbf{A}^p_{kl} = (\ldots, 0, \alpha^{p}_{kl}, 0, \ldots), \quad
    \mathbf{a}^p_{kl} = D\mathfrak{z}(\ww)\left(0, \mathbf{A}^p_{kl}\right), \quad 1 \leq k < l \leq d.
\end{equation}
Here $\alpha^p_{kl}=\alpha_{kl}$ and we have used the superscript $p$ to index depth $p$.

We may now express the pullback metric $\mathfrak{z}^\sharp \Fr$ in the standard basis on $\R^d \times \anti_d^{N+1}$ by computing the 
Frobenius inner product between the $d$ tangent vectors $\mathbf{l}_i$ and the $(N+1)d(d-1)/2$ tangent vectors $\mathbf{a}^p_{kl}$.
Let $I_d$ denote the $d\times d$ identity matrix. We then have
\begin{lemma}
    \label{le:tridiag}
    The standard basis on $\R^d\times \anti_d^{N+1}$ may be ordered such that the pullback metric $\mathfrak{z}^\sharp \Fr$ has the block diagonal structure
\begin{equation}
    \label{eq:pb1}
    h = \left( \begin{array}{cccc} 
    NI_d &   &  &\\
         & h_{a}^{1,2} &  &\\
             &  &\ddots     &  \\
         & & & h_a^{d,d-1}
    \end{array} \right),
\end{equation}
where $h_a^{k,l}$ is the $(N+1)\times (N+1)$ symmetric tridiagonal matrix
\begin{equation}
    \label{eq:pb2}
    h_a^{k,l} = \left( \begin{array}{cccccc} 
    \tfrac{1}{2}(\lambda_k^2 +\lambda_l^2) &  -\lambda_k \lambda_l &  & &\\
      -\lambda_k \lambda_l   & \lambda_k^2 +\lambda_l^2 & -\lambda_k \lambda_l & &\\
             &  -\lambda_k \lambda_l  & \lambda_k^2 +\lambda_l^2  & -\lambda_k \lambda_l  \\  
             &    & \ddots & \ddots  &   \\ \\  &    & -\lambda_k \lambda_l & \lambda_k^2 +\lambda_l^2 &  -\lambda_k \lambda_l \\
         & &  & -\lambda_k \lambda_l & \tfrac{1}{2}(\lambda_k^2 +\lambda_l^2)
    \end{array} \right).
\end{equation}
\end{lemma}
There are $d(d-1)/2$ blocks, each indexed by a basis matrix $\alpha_{kl} \in \anti_d$. This ordering is less intuitive than arranging the metric by depth into $N+1$ blocks, each of size $d(d-1)/2\times d(d-1)/2$.  However, this structure appears naturally in the computation of inner products. It reflects a subtle non-local coupling along the depth of the network that is due to balancedness.

\begin{proof}[Sketch of the proof]
First, rewrite equation~\eqref{eq:basis1} as follows
\begin{equation}
\label{eq:basis1b}
 Q_p^T D\mathfrak{z}(\ww)\left(\theta, \mathbf{A}\right)_p Q_{p-1} = Q_p^TA_p Q_p \Lambda  + \theta -   \Lambda Q_{p-1}^T A_{p-1} Q_{p-1}. 
 \end{equation}
Now observe that the linear transformation of the $p$-th tangent space $\anti_d$ defined by $A_p \mapsto Q_p^T A_p Q_p:=B_p$ is an isometry for the Frobenius norm.
We see immediately that when $\mathbf{A}=\mathbf{0}$ the image 
\begin{equation}
\label{eq:basis4b}
 Q_p^T D\mathfrak{z}(\ww)\left(\theta, \mathbf{0}\right)_p Q_{p-1} =   \theta, 
 \end{equation}
is a collection of $N$ diagonal matrices. On the other hand, when $\theta=0$, 
\begin{equation}
\label{eq:basis4c}
 Q_p^T D\mathfrak{z}(\ww)\left(0, \mathbf{A}\right)_p Q_{p-1} =  B_p \Lambda  -   \Lambda B_{p-1}
 \end{equation}
always vanishes on the diagonal. Thus, the inner-products
\begin{equation}
\label{eq:basis4d}
 \langle  D\mathfrak{z}(\ww)\left(\mathbf{\theta}, \mathbf{0}\right),  D\mathfrak{z}(\ww)\left(0, \mathbf{A}\right) \rangle =0.
\end{equation}
This gives rise to the block $NI_d$ in the upper left corner. 

The only non-zero inner products concern $\mathbf{V}_i= D\mathfrak{z}(0,\mathbf{A}_i)$, $i=1,2$ where
\begin{equation}
    \label{eq:basis6}
\mathbf{A}_1 = (\ldots, 0, \alpha^{p+1}, 0, \ldots), \quad  \mathbf{A}_2 = (\ldots, 0, \alpha^{p}, 0, \ldots),
\end{equation}
have overlapping indices $p+1$ and $p$. In this case, we find that the only non-zero inner product corresponds to $\alpha^{p+1}=\alpha^p$. The main step involves the identity
\begin{equation}
\label{eq:basis5}
    \Tr (\alpha_{ij}\Lambda \alpha_{kl}\Lambda) = \lambda_i \lambda_j \left(\delta_{il}\delta_{jk}-\delta_{ik}\delta_{jl}\right).
\end{equation}

\end{proof}

\subsection{An orthonormal basis for $\mathfrak{z}^\sharp \Fr$.}
\label{subsec:basis}
We may compute an orthonormal basis for $T_{\ww}\balance$ by diagonalizing the block tridiagonal matrix of Lemma~\ref{le:tridiag}. The diagonalization procedure uses the theory of Jacobi matrices, in particular the fact that the eigenbasis of each block $h_a^{k,l}$ may be computed using Chebyshev polynomials. 

The results of this approach may be summarized as follows. Assume that $\ww = \mathfrak{z}(\Lambda,Q_N,\ldots, Q_0)$ and let us denote the columns of each $Q_p$ by
\begin{equation}
    \label{eq:cheb1}
    Q_p = \left( \begin{array}{cccc}
    \uparrow & \uparrow  & \cdots & \uparrow \\
          q_{p,1} & q_{p,2} & \cdots & q_{p,d} \\
     \downarrow & \downarrow  & \cdots & \downarrow     \end{array}\right).
\end{equation}
We will define a collection of vectors denoted by
\begin{eqnarray}
    \label{eq:vb1}
    \mathbf{l}^{k} &=& (l_N^{k},\cdots, l_{1}^{k}), \quad 1\leq k \leq d;\\
    \label{eq:vb2}
    \mathbf{u}^{k,l,p} &=& (u_N^{k,l,p},\cdots, u_{1}^{k,l,p}), \quad 1\leq k < l\leq d, \quad 0 \leq p \leq N. 
\end{eqnarray}
There are $d$ vectors of type $\mathbf{l}$ and $(N+1)d(d-1)/2$ vectors of type $\mathbf{u}$. The coordinates of these vectors are as follows. First, for the $\mathbf{l}$ vectors 
\begin{equation}
\label{eq:vb3}
    l_s^{k} = \frac{1}{\sqrt{N}} q_{s,k} q_{s-1,k}^T, \quad 1\leq s \leq N.  
\end{equation}

We next consider the $\mathbf{u}$ vectors for $1\leq p \leq N-1$. This range for $p$ corresponds to  the overparametrization freedom $\od^{N-1}$. We define
\begin{equation}
\label{eq:vb4}
    u_s^{k,l,p}  =  a^{k,l,p,s} q_{s,k} q_{s-1,l}^T + a^{l,k,p,s} q_{s,l} q_{s-1,k}^T ,   \quad 1\leq s \leq N. 
\end{equation}
Here we have denoted for brevity
\begin{equation}
\label{eq:vb5}
    a^{k,l,p,s}  = \sqrt{\frac{1}{N\left(\lambda_k^2+\lambda_l^2-2\lambda_k\lambda_l \cos 
    \tfrac{p\pi}{N}\right)}} \left(\lambda_k \sin \frac{(s-1)p\pi}{N}- \lambda_l \sin \frac{sp\pi}{N}\right).  
\end{equation}
Finally, we consider the action of the matrices $Q_0$ and $Q_N$ that generate the SVD of $W$. First, for $p=0$ define
\begin{equation}
\label{eq:vb6}
    u_s^{k,l,0}  = \sqrt{\frac{\lambda_k^2-\lambda_l^2}{\lambda_k^{2N}-\lambda_l^{2N}}} \lambda_k^{s-1}\lambda_l^{N-s} q_{s,l} q_{s-1,k}^T.   
\end{equation}
Similarly, define for $p=N$
\begin{equation}
\label{eq:vb7}
    u_s^{k,l,N}  = \sqrt{\frac{\lambda_k^2-\lambda_l^2}{\lambda_k^{2N}-\lambda_l^{2N}}} \lambda_k^{N-s}\lambda_l^{s-1} q_{s,k} q_{s-1,l}^T.   
\end{equation}
We then have 
\begin{theorem}[Menon,Yu~\cite{MY-dln}]
    \label{thm:embed} The vectors $(\mathbf{l},\mathbf{u})$ defined in equations~\eqref{eq:vb3}--\eqref{eq:vb7} form an orthonormal basis for $(T_\ww\balance,\Fr)$.
\end{theorem}

Theorem~\ref{thm:embed} is the main tool in the proofs of Theorem~\ref{thm:entropy} and Theorem~\ref{thm:submer}. Let us first explain how Theorem~\ref{thm:submer} is obtained from it. 

In order to show that the  metric $g^N$ on $\Mdd$ is obtained by Riemannian submersion of $\balance_r$ under $\phi$ we must compute $\mathrm{Ker}\,\phi_*$ and find an isometry between $(\mathrm{Ker}\,\phi_*)^\perp$ and $T_W\Mdd$. The kernel $\mathrm{Ker}\,\phi_*$ corresponds to the group action of $\od^{N-1}$ and it is the span $\mathbf{u}^{k,l,p}$ for $1\leq p\leq N-1$. Thus, $(\mathrm{Ker}\,\phi_*)^\perp$ is spanned by the vectors $\mathbf{l}$, $\mathbf{u}^{k,l,0}$, and $\mathbf{u}^{k,l,N}$. We see immediately that these correspond exactly to the eigendecomposition of $\mathcal{A}$ computed in Lemma~\ref{le:diagon}. This establishes that $(\Mdd,g^N)$ is obtained by Riemannian submersion from $(\balance,\Fr)$.

Theorem~\ref{thm:entropy} is obtained as follows. We must compute the volume of the $\od^{N-1}$ orbit $\phi^{-1}(W) \subset \balance$. We work in local coordinates given by our parametrization, and observe that fixing $W$ corresponds to fixing $\Lambda$, $Q_0$ and $Q_N$. The pullback metric $\mathfrak{z}^\sharp \iota$ is invariant under the $\od^{N-1}$ action. This reduces the computation of the group volume to the computation of the determinant of $\mathfrak{z}^\sharp \iota$ restricted to the subspace spanned by $\mathbf{u}^{k,l,p}$, $1\leq p \leq N-1$. This determinant may be computed explicitly using Theorem~\ref{thm:embed}, yielding Theorem~\ref{thm:entropy}. 

The proof for the rank-deficient case requires more care, but the above arguments captures the essence of these theorems.  Further details may be found in~\cite{MY-dln}.

\section{Cartoons}
\label{sec:cartoons}
 The way dynamicists actually work is by drawing pictures that capture different forms of geometric intuition. The purpose of this section is to illustrate this way of thinking, complementing several theorems with an impressionistic visualization of the theorems. A good picture can tell us what a theorem means; or better yet, help us guess what theorems we should be proving. 
 
 The DLN is an attractive model because it provides us with a rich set of pictures that suggest new questions for study. We organize our cartoons visually to bring out these aspects, focusing on the following themes:
  \begin{enumerate} 
 \item {\em Riemannian submersion.\/} We often organize images into a $\ww$-space upstairs and $W$-space downstairs. Riemannian submersion arises in several areas of analysis, especially mass transportation, and it is helpful to study the effects of overparametrization in analogy with these areas. 
 
 \item {\em Conic sections.\/} The equations~\eqref{eq:b1} that define the balanced varieties are {\em quadratic\/} equations. Thus, when drawing a two or three dimensional sketch, we caricature the balanced varieties by conic sections. This caricature then immediately focuses our attention on the balanced variety $\balance_{\mathbf{0}}$.
 \item {\em Foliation by rank. \/} The balanced variety is itself foliated by rank. We have only studied the case $d=r$ in depth, but the rank-deficient case $r < d$ is of great interest. This is harder to visualize, so we have demonstrated the effect of rank within the zero energy set for an example discussed in equations~\eqref{eq:mc2}--\eqref{eq:mc5}.
 \item {\em Foliation by group action. \/} The parametrizations by the polar factorization and SVD provide a different way of exploring balanced manifolds with group action. The computation of the pullback metric illustrates the subtle nature of the coupling by depth along the network. We visualize this by slicing the balanced manifold.
 \item {\em Mean curvature is an \Ito\/ correction. \/} We often use the fact that the mean curvature from of an embedded manifold is the `deterministic push' that arises from isotropic tangential stochastic fluctuations. This cartoon allows us to introduce Riemannian Langevin equations for the DLN in an intuitive manner. This can be illustrated rather simply and then generalized to more sophisticated foliations, such as in Section~\ref{subsec:dyson}.
\end{enumerate}

\begin{figure}
    \centering
    \includegraphics[scale=0.2]{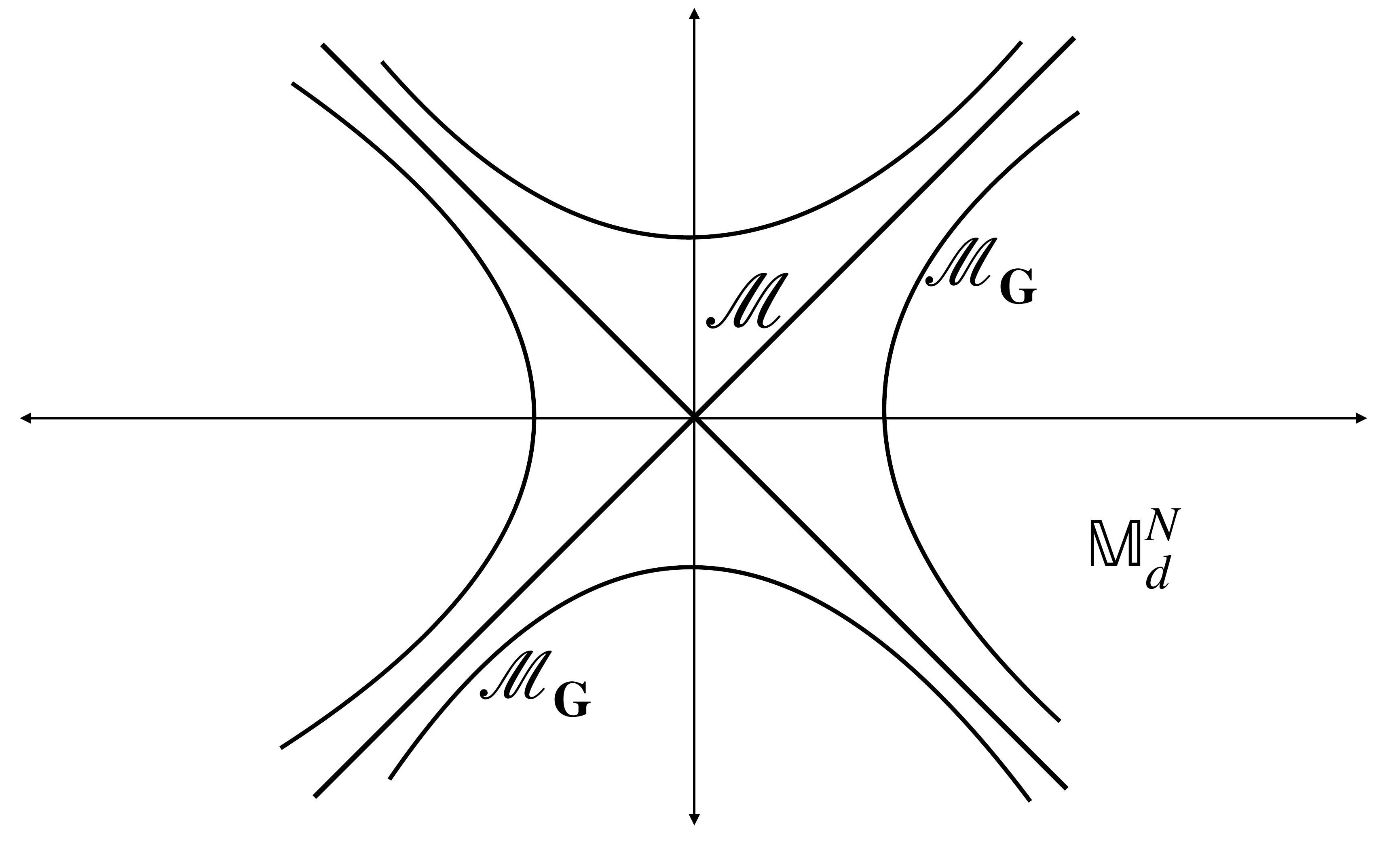}
    \caption{The foliation of $\Md^N$ by the balanced varieties $\balance_{\mathbf{G}}$ may be visualized as a foliation by conic sections (see equation~\eqref{eq:b6q}).}
    \label{fig:geom1}
\end{figure}
\begin{figure}
    \centering
    \includegraphics[scale=0.2]{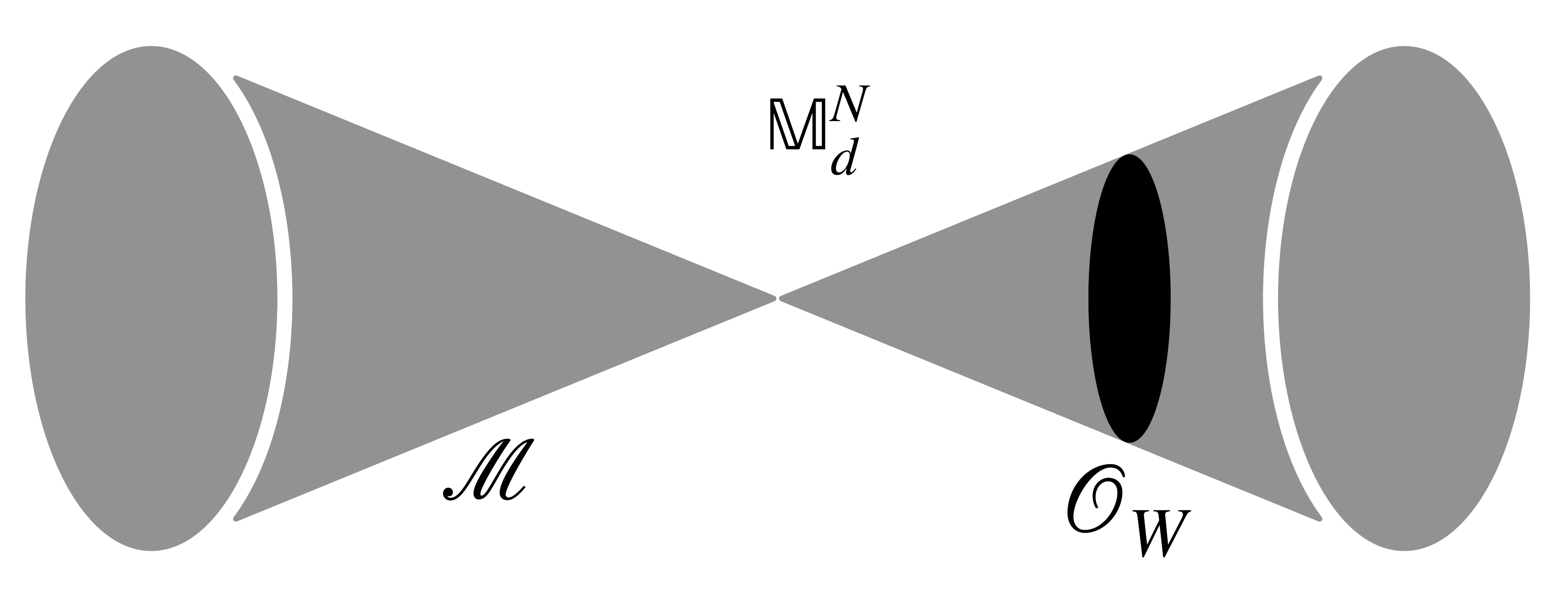}
    \caption{In comparison with Figure~\ref{fig:geom1}, we now blow up the balanced manifold $\balance$ and visualize the foliation into group orbits $\orbitx$ by slicing $\balance$.}
    \label{fig:geom2}
\end{figure}
\begin{figure}
    \centering
    \includegraphics[scale=0.2]{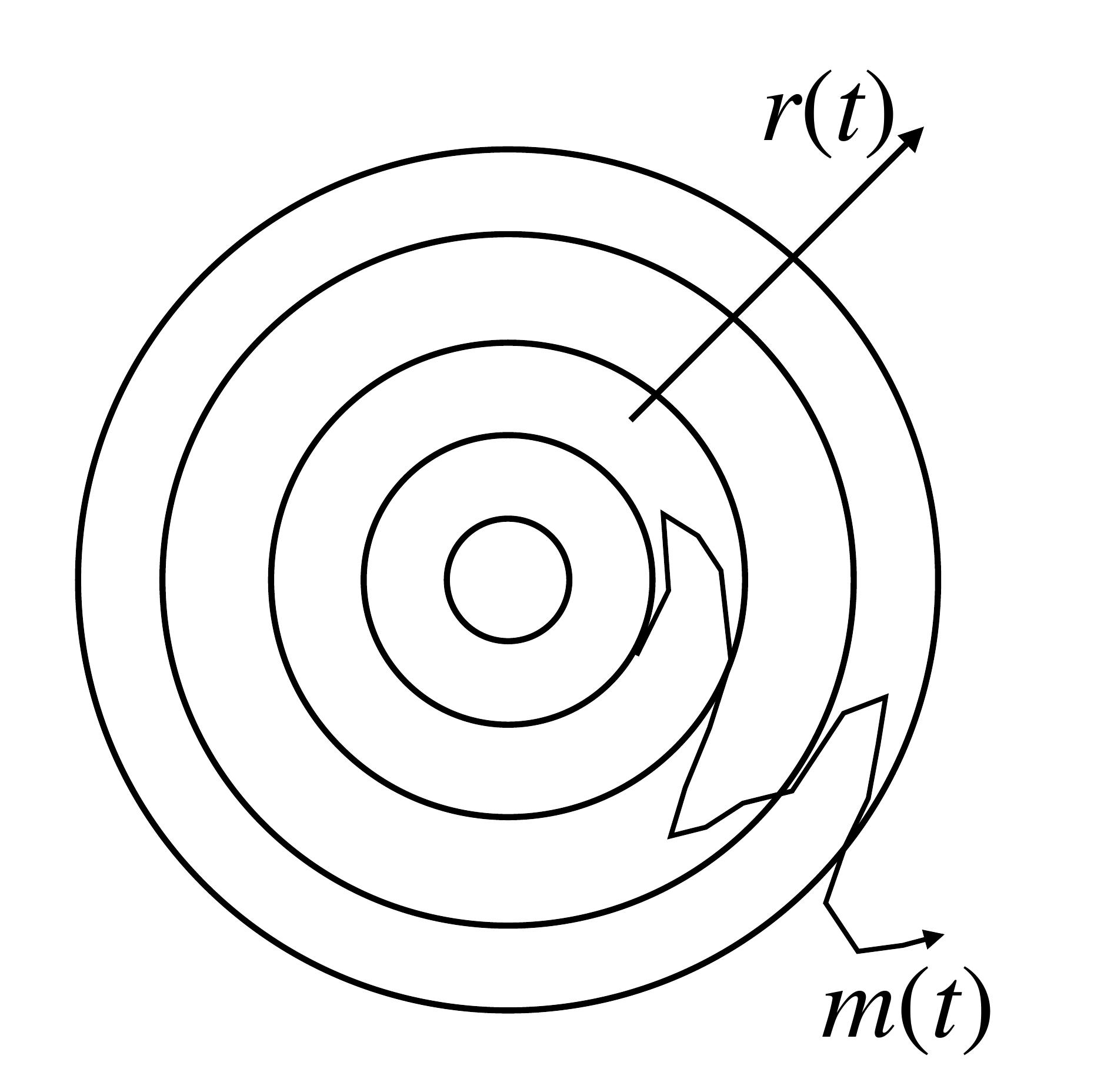}
    \caption{Motion by (minus one half) curvature arising from tangential Brownian fluctuations as discussed in Section~\ref{subsec:spheres}.}
    \label{fig:geom3}
\end{figure}
\begin{figure}
    \centering
    \includegraphics[scale=0.2]{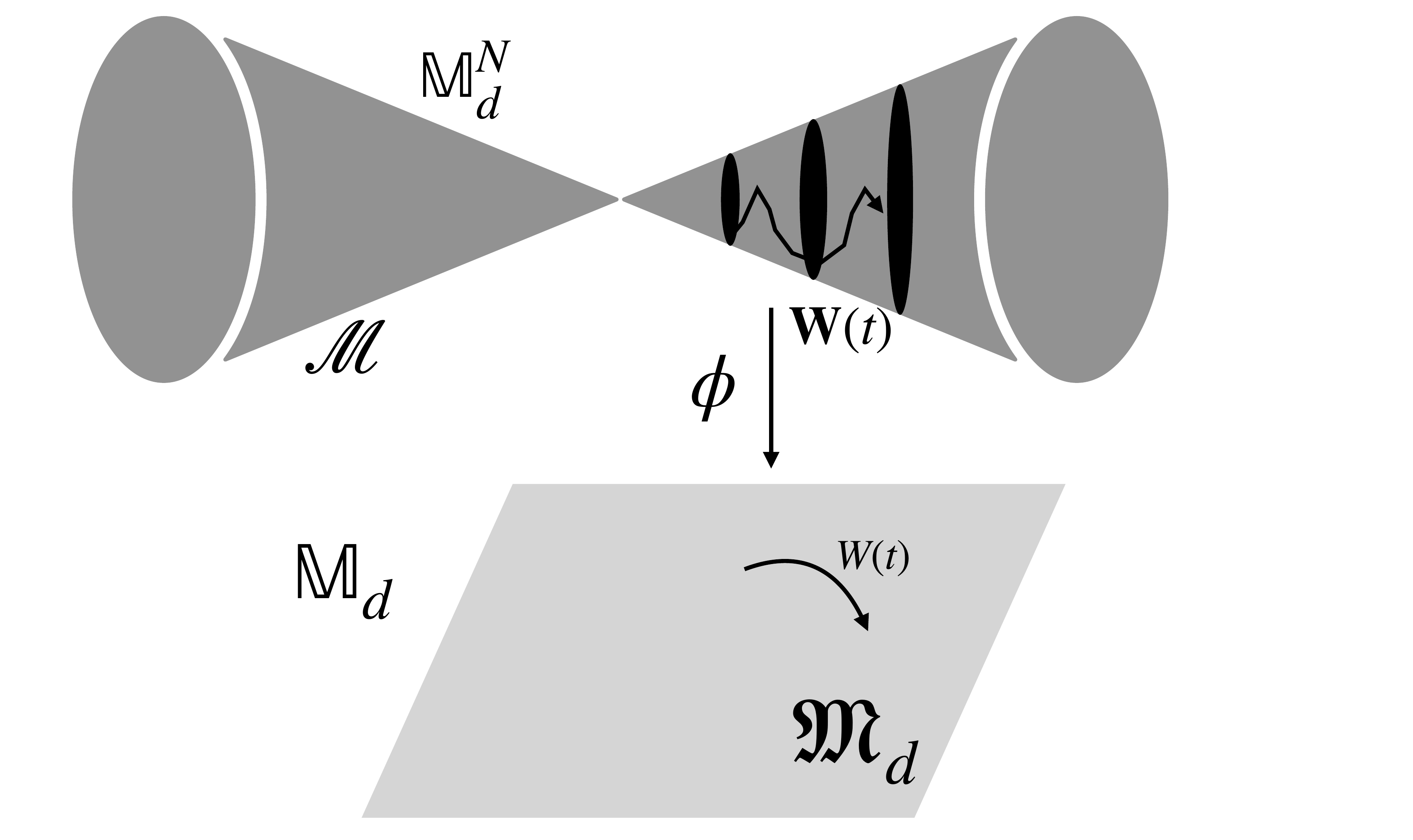}
    \caption{Riemannian submersion  $\phi:\balance \to \Mdd$ and the dynamics upstairs and downstairs. In this image, we illustrate the RLE with stochastic dynamics upstairs and deterministic gradient descent of free energy downstairs. See equations~\eqref{eq:rle-up2} and~\eqref{eq:rle-down2}.}
    \label{fig:geom4}
\end{figure}
\begin{figure}
    \centering
    \includegraphics[scale=0.2]{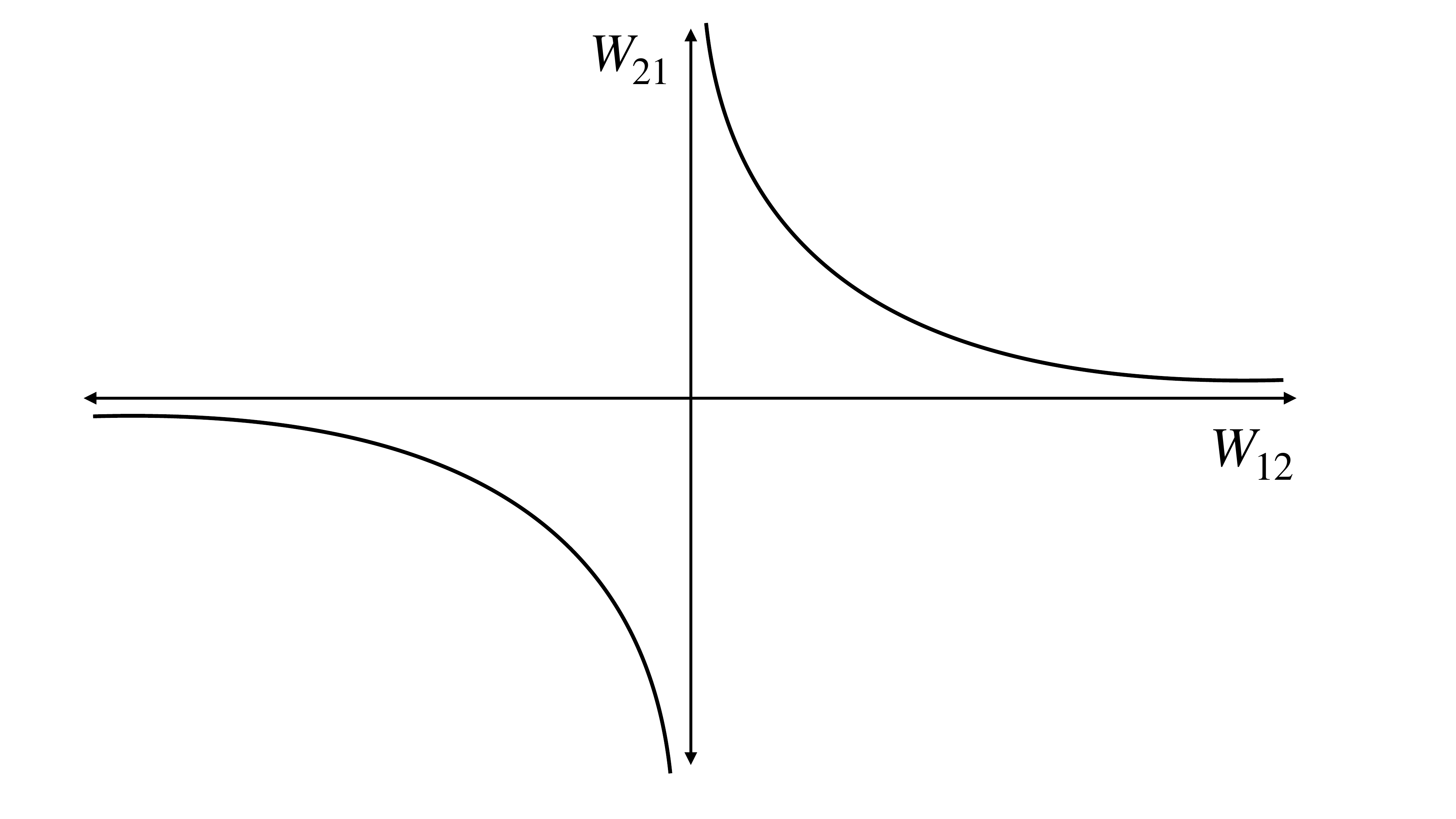}
    \caption{A rank-one variety within the zero energy set for matrix completion. See equations~\eqref{eq:mc4}--\eqref{eq:mc5}.}
    \label{fig:geom5}
\end{figure}

\section{The Riemannian Langevin equation (RLE)}
\label{sec:RLE}
\subsection{Overview} The next three sections address the following question: can we use the DLN to develop the thermodynamics of deep learning? 

In the spirit of this article, such a framework must be rooted in the geometry of overparametrization. We have seen that the analysis of the DLN leads naturally to the Riemannian manifolds $(\balance_r,\iota)$ and $(\Mr,g^N)$. Further, we have used this geometry to define a Boltzmann entropy as well as the gradient descent of free energy (see Section~\ref{sec:entropy}). In the physical analogy with thermodynamics, what we are now seeking is an explicit understanding of microscopic fluctuations and Gibbs measures. Our task, therefore, is to develop the underlying `statistical physics' that corresponds to the entropy formula, based on strictly geometric foundations. 

 In this section, we introduce the Riemannian Langevin Equation (RLE) as the natural geometric model for fluctuations. Our argument is by mathematical analogy and it takes the following form. First, we review the Langevin equation on $\R^d$. We then explain how the RLE is its natural extension to Riemannian manifolds. This is followed by the description of an important example from random matrix theory (Dyson Brownian motion). We then abstract the underlying structure combining Riemannian submersion with the theory of Brownian motion on manifolds. Finally, in Section~\ref{subsec:RLE-dln}, we describe the explicit nature of the RLE in the DLN. 

The correspondence between the RLE for the DLN and Dyson Brownian motion is a powerful tool. It provides a way to use ideas from random matrix theory to analyze training dynamics of the DLN. It immediately leads to open questions of intrinsic mathematical interest which are discussed in Section~\ref{sec:open}. Algorithmically, the RLE replaces the problem of minimizing the free energy with Gibbs sampling. 

The RLE is a generalization of the Langevin equation based on the intrinsic Riemannian geometry. We introduce it here as a phenomenological model that incorporates the effect of noise, revealing the interplay of noise and curvature. Whether a mathematical theory works in pratice is a different matter altogether. Our RLE loosely corresponds to how noise due to round-off errors may interact with the geometry of overparametrization. It does not account for the many different ways in which noise actually arises in training dynamics (discretization, batch processing, random initialization,\ldots). Nevertheless, we hope the structure of the RLE will provide a useful framework for  practitioners, since it is a reference model which is a natural stochastic analogue of gradient descent (though it is {\em not\/} stochastic gradient descent (SGD) corresponding to batch processing).

\subsection{The Langevin equation}
First, let us recall the Langevin equation on $\R^d$. Assume given $\beta>0$ and a potential $V:\R^d \to \R$. The Gibbs measure $\mu_\beta$ with respect to this potential is the probability measure with density
\begin{equation}
    \label{eq:rle1}
    \rho_\beta(x) = \frac{e^{-\beta V(x)}}{Z_\beta}, \quad Z_\beta = \int_{\R^d} e^{-\beta V(x')} \, dx'.
\end{equation}
The Langevin equation describes microscopic stochastic fluctuations in and out of equilibrium. It is the \Ito\/ SDE~\footnote{We suppress the usual subscript $t$ (such as $dx_t=-\nabla \,V(x_t)\, dt + \sqrt{2/\beta}\, dW_t$) in the notation to simplify the SDE for coordinates and matrices.}
\begin{equation}
    \label{eq:rle2}
    dx=  -\nabla V(x) \, dt + \sqrt{\frac{2}{\beta}} \, dW
\end{equation}
where $W_t$ denotes the standard Wiener process in $\R^d$. 

Assume the law of $x_t$ has a density $\rho(t,\cdot)$ and assume that the law of $x_0$ is given. 
The evolution of the probability density $\rho$ is then given by the Fokker-Planck (or forward Kolmogorov) equation
\begin{equation}
    \label{eq:rle3}
    \partial_t \rho =  \frac{1}{\beta} \triangle \rho + \nabla \cdot \left (\rho\, \nabla V(x_t)\right).
\end{equation}
Under suitable assumptions on $V$, $\lim_{t \to \infty} \rho(t,\cdot)$ is the Gibbs density $\rho_\beta$ defined in equation~\eqref{eq:rle1}.

\subsection{RLE: general principles.\/}
The Langevin equation and the Fokker-Planck equation have been widely studied in mathematical physics. Our interest here lies in the generalization of these equation to Riemannian manifolds. The main issue is to understand how to replace the noise in equation~\eqref{eq:rle2}.

Assume given a Riemannian manifold $(\mathcal{N},h)$. We define the Laplace-Beltrami operator $\triangle_h$ by its action on a smooth function $f:\mathcal{N}\to \R$ in coordinates by
\begin{equation}
\label{eq:laplace}
\triangle_h f = \frac{1}{\sqrt{|h|}} \partial_{x^i} \left( \sqrt{|h|} h^{ij} \partial_{x^j} f \right),
\end{equation}
where  $|h|$ denotes the determinant of $h_{ij}$ and $h^{ij}$ denotes its inverse. Brownian motion on $(\mathcal{N},h)$ at inverse temperature $\beta>0$ is a diffusion on $\mathcal{N}$ whose generator is $\tfrac{1}{\beta}\triangle_h$. While the parameter $\beta$ can be included within the metric, we will include it separately, since there are many situations where it is necessary to hold $h$ fixed and rescale the strength of the noise.

The abstract characterization of Brownian motion as a diffusion process with a generator is powerful. However, in practice (for example, for simulation) we seek `hands-on' constructions of Brownian motion as the limit of suitable random walks. There are several such constructions of Brownian motion using Stratonovich SDE (see for example~\cite[\S 3.2]{Hsu},~\cite[\S V.4]{Ikeda},~\cite[Thm 3]{IM}). Once Brownian motion on $(\mathcal{N},h)$ has been constructed with SDE, it is easy to (formally) extend the Langevin equation~\eqref{eq:rle2} to the Riemannian setting. 

Let $B_t^{\beta,h}$ denote Brownian motion on $(\mathcal{N},h)$ at inverse temperature $\beta>0$ and consider the (formal) \Ito\/ SDE
\begin{equation}
    \label{eq:rle4}
    dx=  -\mathrm{grad}_h V(x) \, dt + dB^{\beta,h}.
\end{equation}
This equation is only formal, because SDE on manifolds must be defined using the Stratonovich formulation to ensure coordinate independence. However, equation~\eqref{eq:rle4} makes the analogy with~\eqref{eq:rle2} transparent and one may recover the `true' Stratonovich SDE by the inclusion of a drift term computed using the \Ito-Stratonovich correction formula. The  important geometric aspect for matrix manifolds is that this yields a curvature correction that is often explicitly computable.


\section{Curvature and entropy: examples}
\subsection{Dyson Brownian motion via Riemannian submersion}
\label{subsec:dyson}
Dyson Brownian motion at inverse temperature $\beta>0$ is the interacting particle system described by the \Ito\/ SDE
\begin{equation}
    \label{eq:dyson}
    dx_i = \sum_{j \neq i}\frac{1}{x_i-x_j}\, dt + \sqrt{\frac{2}{\beta}} dW_i, \quad 1\leq i \leq d.
\end{equation}
Here $W_t=(W_1,\ldots,W_d)_t$ denotes the standard Wiener process in $\R^d$. This equation has a unique strong solution when $\beta\geq 1$ that remains confined to the Weyl chamber 
\begin{equation}
    \label{eq:rle5}
    \weyl_d = \{ x \in \R^d \left| x_1 < x_2 < \ldots < x_d \right. \}.
\end{equation}
In~\cite{HIM23} we presented a geometric construction of Dyson Brownian motion using Riemannian submersion. 
This construction has a natural extension to the DLN, but we must first introduce some notation to explain these ideas. 

Let $\Her_d$ denote the space of $d\times d$ Hermitian matrices equipped with the norm $\|M \|^2 =\Tr(M^*M)$ and let $U_d$ denote the unitary group. Given $x \in \weyl_d$, let $X=\diag(x)$, and let $\mathcal{O}_x$ denote the isospectral orbit 
\begin{equation}
    \label{eq:isospec}
    \mathcal{O}_x=\{M \in \Her_d \left| M = UXU^*, \; U \in U_d\right. \}.
\end{equation}

We observe that the space $\Her_d$ is foliated by $U_d$ orbits, each of which is an isospectral set, and an isospectral manifold when the elements of $X$ are distinct. The main insight in~\cite{HIM23} is the interplay between mean curvature, tangential noise, and the entropy that is captured in the following `lift upstairs' of equation~\eqref{eq:dyson}. 

Let $M \in \mathcal{O}_x$ and let $P_M$ and $P_M^\perp$ denote the orthonormal projections onto $T_M\mathcal{O}_x$ and  $(T_M\mathcal{O}_x)^\perp$ respectively with inner-products computed according to $\|\cdot \|^2$.
Let $H_t$ denote the standard Wiener process on $(\Her_d,\|\cdot\|^2)$. For every $\beta>0$ we define the \Ito\/ SDE
\begin{equation}
    \label{eq:proj1}
    dM= P_{M} dH + \sqrt{\frac{2}{\beta}} P_{M}^\perp dH.
\end{equation}
Assume that $x_t$ is the unique strong solution to equation~\eqref{eq:dyson} for its maximal interval of existence $[0, T_{\max})$ (this is $[0,+\infty)$ when $\beta \geq 1$).
We then have the following
\begin{theorem}[Huang, Inauen, Menon~\cite{HIM23}]
\label{thm:dyson}
\begin{enumerate}
\item[(a)] The eigenvalues of $M_t$ have the same law as the solution $x_t$ to~\eqref{eq:dyson} for $t \in [0, T_{\max})$. 
\item[(b)] When $\beta=+\infty$, the group orbits $\mathcal{O}_{x_t}$ evolve by motion by (minus one half) mean curvature.
\end{enumerate}
\end{theorem}
The main insight in the theorem is captured in the case $\beta=+\infty$. In this setting, the stochastic fluctuation $P_M dH$ is {\em tangential\/} to the orbit $\mathcal{O}_x$. However, the \Ito\/ correction is minus a half times the mean curvature. In particular, it is deterministic and {\em normal\/} to $\mathcal{O}_x$.

\begin{remark}
\label{rem:dyson-entropy}
The entropy formula in Theorem~\ref{thm:entropy} is inspired by the analogous formula for $\mathcal{O}_x$:
\begin{equation}
    \label{eq:dyson2}
    S(x) = \log \mathrm{vol} (\mathcal{O}_x) = 2 \log \mathrm{van}(x) + C_d, 
\end{equation}
where $C_d$ is a universal constant. We then see that we may rewrite equation~\eqref{eq:dyson} as follows~\footnote{In~\cite{HIM23}, we work with Hermitian matrices, which is the most natural setting for Dyson Brownian motions. The factors of $2$ that cancel are included for consistency with the general theory of RLE.}
\begin{equation}
    \label{eq:dyson2b}
    dx = \frac{1}{2} \nabla S(x) dt + \sqrt{\frac{2}{\beta}} dW. 
\end{equation}
\end{remark}

\begin{remark}
\label{rem:kappa}
    The role of $\beta$ in equation~\eqref{eq:dyson} reflects the standard terminology of random matrix theory. In equation~\eqref{eq:proj1}, however, the parameter $\beta$ reflects an anisotropic splitting of the inner-product in $\Her_d$ between the tangent space $T_M\mathcal{O}_x$ and $T_M\mathcal{O}_x^\perp$. For these reasons, we will introduce a separate parameter $\kappa>0$  to describe the anisotropy, and use $\beta>0$ for the `true' inverse temperature when augmenting equation~\eqref{eq:proj1} with a loss function. 
\end{remark}  

In order to define stochastic gradient descent by the RLE, we include a loss function in equation~\eqref{eq:proj1} as follows.     Given $E: \weyl \to \R$, set $L(M)=E(\mathrm{eig}(M))$ and consider
\begin{equation}
    \label{eq:proj2}
dM = -\nabla_{M} L(M) \, dt +  \sqrt\frac{2}{\beta} \left( P_{M} dH + \sqrt{\kappa} P_{M}^\perp dH\right).
\end{equation}
The corresponding equation downstairs is
\begin{equation}
    \label{eq:proj3}
dx = -\nabla_{x} F_\beta(x) dt +  \sqrt\frac{2\kappa}{\beta} dW,
\end{equation}
where $W_t$ is the standard Wiener process in $\R^d$ and 
\begin{equation}
    \label{eq:proj3b} F_\beta(x) = E(x) -\frac{1}{\beta}S(x)
\end{equation} 
is the free energy. When $\kappa=0$, we have gradient descent of free energy. 
In equations~\eqref{eq:proj2} and~\eqref{eq:proj3} the gradient operator $\nabla$ is with respect to the standard inner products on $\Her_d$ and $\R^d$ respectively. In fact, these metrics are related by Riemannian submersion.

In what follows, we will use equation~\eqref{eq:proj2} as a model for stochastic gradient descent by the RLE of a loss function, which accounts separately for anisotropic fluctuations in the gauge group and the observable. 

\subsection{The stochastic origin of motion by mean curvature.\/}
\label{subsec:spheres}
The appearance of mean curvature in Theorem~\ref{thm:dyson}(b) may be understood through a simpler example.
Let $m \in \R^d$, let $r=|m|= \sqrt{m^Tm}$ denote the length of $m$, and let $S^{d-1}_r$ denote the sphere of radius $r$ in $\R^d$. Let $P_m$ denote the orthonormal projection onto $T_m S^{d-1}_r$; explicitly, we have
\[P_m = I_d - \frac{mm^T}{|m|^2}.\]
Let $W_t$ denote the standard Wiener process in $\R^d$ and consider the \Ito\/ SDE
\begin{equation}
    \label{eq:circle}
    dm = P_m \,dW.
\end{equation}
This SDE is the continuum limit of a random walk which may be intuitively described as `at each time step, take an isotropically distributed random spatial step in the tangent space'. 

Stochastic calculus makes it easy to capture the intuitive content of such clumsy verbal descriptions. In particular, since the SDE is in the \Ito\/ form, it gives rise to a correction in $r_t$. The reader is invited to apply \Ito's formula to see that while $m_t$ always evolves by {\em tangential\/} stochastic fluctuations governed by equation~\eqref{eq:circle}, the radius $r_t=|m_t|$ satisfies
\begin{equation}
    \label{eq:circle2}
    \dot{r} = \frac{d-1}{2r}.
\end{equation}
The standard normalization of the mean curvature of the sphere $S_r^{d-1}$ is such that it has magnitude $(d-1)/r$ and points inwards. Consequently, equation~\eqref{eq:circle2} tells us that the concentric spheres $S_r^{d-1}$ evolve by motion by minus a half times mean curvature. 

Theorem~\ref{thm:dyson}(b) is a more sophisticated instance of the same interplay. The \Ito\/ correction for the eigenvalues of $M_t$ are computed using the first and second-order variation formulas for eigenvalues (see~\cite{HIM23}). 

\section{RLE for stochastic gradient descent}
\subsection{Riemannian submersion with a group action}
\label{subsec:RLE-general}
In this section, we introduce Riemannian Langevin equations to model stochastic gradient descent of free energy. We formalize the insights obtained in the examples of Sections~\ref{subsec:dyson}--~\ref{subsec:spheres} in a general geometric framework. We also include `lifted' loss functions as in deep learning. These RLE provide stochastic extensions of gradient flows such as~\eqref{eq:big-grad1a}, which are consistent with the underlying geometry. The tunable parameter $\kappa>0$ modulates the relative strength of the noise in null directions, or equivalently the gauge group, with the noise in the observable. When $\kappa=0$, we recover Riemannian gradient descent of free energy for the observable.

The geometric assumptions are as follows. We assume given a reference Riemannian manifold $(\refspace,\refmetric)$ and a Lie group $\refgroup$ of isometries. We then consider the quotient space $\quotientspace =\refspace/\refgroup$ equipped with the metric $\quotientmetric$ given via Riemannian submersion. Let $\phi: \refspace \to \quotientspace$ denote the submersion map and let $\orbit_x$ denote the inverse image $\phi^{-1}(x)$ for each $x \in \quotientspace$. Then the following general principles hold:
\begin{enumerate}
    \item There is always a natural Boltzmann entropy $S(x)=\log \mathrm{vol}(\orbit_x)$. 
    \item The gradient in $(\refspace,\refmetric)$ of the entropy $S(\phi(m))$ is the mean curvature at each point $m \in \orbit_x$.
    \item Brownian motion on $(\refspace,\refmetric)$ projects to Brownian motion on $(\quotientspace,\quotientmetric)$. Precisely, if $M^\beta_t$ is a diffusion in $\refspace$ with generator $\beta^{-1}\triangle_\refmetric$ then $\phi(M^\beta_t)$ is a diffusion on $\quotientspace$ with  generator $\beta^{-1}\triangle_\quotientmetric$.
\end{enumerate}
 
The tangent space $T_m\refspace$ at each point $m \in \orbit_x$ splits into $T_m \orbit_x$ and its orthogonal complement $T_m \orbit_x^\perp$. Let $P_m$ and $P_m^\perp$ be the orthonormal projections onto these subspaces of $T_m \orbit_x$. The formal analog of the anisotropic splitting in equation~\eqref{eq:proj2} is 
\begin{equation}
    \label{eq:split1}dm^{\beta,\kappa} = P_m dM^\beta + \sqrt{\kappa} P_m^\perp dM^\beta. 
\end{equation} 
The use of stochastic calculus makes it easy to describe the intuitive basis of the anisotropic decomposition and we will use this expression below for its simplicity and suggestive power. Such \Ito\/ calculus expressions may be made rigorous in two ways. First, we may use the equivalent Stratonovich SDE, computing the required \Ito-Stratonovich correction. 
Second, it is also possible to avoid stochastic calculus by replacing the standard Laplacian with an anisotropic Laplacian as follows.

Rigorously, the decomposition~\eqref{eq:split1} can be expressed using the generator of the process $m^\beta_t$. The orthogonal decomposition $T_m\refspace = \T_m\orbit_x \oplus T_m \orbit_x^\perp$ allows us to split the Laplacian $\triangle_\refmetric$ into `angular' and `radial' Laplacians, denoted by $\triangle_{\orbit}$ and $ \triangle_{\orbit^\perp}$ respectively. Then the anisotropic process $m^\beta_t$ is the diffusion with generator 
\begin{equation}
    \label{eq:split2} \frac{1}{\beta} \left( \triangle_{\orbit} + \kappa  \triangle_{\orbit^\perp}\right).
\end{equation} 
An explicit description of an orthonormal basis for $T_m\refspace$ is very useful when we apply this idea in practice. In particular, the basis in Section~\ref{subsec:basis} is used when we apply these ideas to the DLN. 

The energetics are as follows. We assume given a loss function $E:\quotientspace \to \R$ downstairs that we lift to a loss function $L: \refspace \to \R$, $L = E \circ \phi$. 

We then define {\em stochastic gradient descent by the RLE\/} of the loss function through the  (formal) \Ito\/ SDE `upstairs'
\begin{equation}
    \label{eq:rle-up}
    dm^{\beta,\kappa} = -\mathrm{grad}_\refmetric L(m)\, dt + P_{m} dM^\beta + \sqrt{\kappa} \, P_m^\perp dM^\beta.
\end{equation}
This equation is an analogue of~\eqref{eq:proj2}. We then expect in analogy with equation~\eqref{eq:proj3} that the corresponding SDE `downstairs' is
\begin{equation}
    \label{eq:rle-down}
    dx = -\mathrm{grad}_\quotientmetric F_\beta (x)\, dt +  dX^{\beta/\kappa}, 
\end{equation}
with free energy
\begin{equation}
    \label{eq:rle-free-energy}
    F_\beta(x) = L(x) - \frac{1}{\beta} S(x). 
\end{equation} 
In the limit $\kappa=0$, we have the {\em stochastic\/} flow upstairs
\begin{equation}
    \label{eq:rle-up2}
    dm = -\mathrm{grad}_\refmetric L(m)\, dt + P_{m} dM^\beta. 
\end{equation}
The flow of $m_t$ projects  to the {\em deterministic\/} gradient flow downstairs
\begin{equation}
    \label{eq:rle-down2}
    \dot{x} = -\mathrm{grad}_\quotientmetric F_\beta (x).
\end{equation}

We have thus obtained a strictly geometric model for the thermodynamic concept of quasistatic equilibration. The observable $x_t$ has no fluctuations and evolves according to the gradient descent of free energy. However, $m_t$ upstairs stochastically  `rolls without slipping' along a mean path with drift $\mathrm{grad}_\refmetric L$.~\footnote{The terminology `rolling without slipping' refers to the kinematics of the wheel. When the center of a wheel moves at steady velocity, the point of contact with the ground always has instantaneous velocity zero (and thus does not `slip'). This idea can be used to define Brownian motion on any manifold with a connection~\cite{Hsu,Ikeda}.}

\subsection{RLE for the DLN}
\label{subsec:RLE-dln}
We now apply these general principles to the DLN. The reference  Riemannian manifold $(\refspace,\refmetric)$ is the balanced manifold $(\balance,\Fr)$. The Lie group of isometries is $\od^{N-1}$, and the quotient manifold is $(\Mdd,g^N)$.

Let $\bbb_t$ denote Brownian motion on $(\balance,\Fr)$. We may construct $\bbb_t$ by projecting standard Brownian motion on $\Md^N$ onto $\balance$, or by pushing forward Brownian motion on the parameter space under the parametrization $\mathfrak{z}$. The inverse temperature $\beta$ may be included by a trivial scaling so that we have the process $\bbb_t^\beta$. We further split $\bbb^\beta$ into two processes using the orthogonal projection onto the group orbit $\orbit_W \subset \balance$, obtaining the process $\bbb_t^{\beta,\kappa} \in \balance$. This process is analogous to $m^{\beta,\kappa}$ defined in equation~\eqref{eq:split1}, though we use slightly different notation for convenience.

Theorem~\ref{thm:BM-dln} below provides an explicit description of Brownian motion downstairs on $(\Mdd,g^N)$. We must first introduce some notation to make the theorem transparent. Recall from equation~\eqref{eq:b12} that $W=Q_N \Sigma Q_0^T$ is the SVD of $W$ and that $\Lambda= \Sigma^{1/N}$. We also define two diagonal matrices obtained by differentiation from $\Sigma$. Let $\Sigma'$ be the diagonal matrix with $k$-th entry
\begin{equation}
    \label{eq:sigma-prime}
    \Sigma'_{kk} = \sum_{l \neq k} \left( \frac{N\lambda_k^{2N-1}} {\lambda_k^{2N}-\lambda_l^{2N}} - \frac{\lambda_k}{\lambda_k^{2}-\lambda_l^{2}}\right) \lambda_k^{N-1}, \quad 1\leq k \leq d.
\end{equation}
Similarly, define the matrix $\Sigma''$ to be the diagonal matrix with entries
\begin{equation}
    \label{eq:sigma-double-prime}
    \Sigma''_{kk} = (N-1)\lambda_k^{N-2}, \quad 1\leq k \leq d.
\end{equation}

Assume given a matrix $B_t$ valued Brownian motion in $\Md$, or equivalently, $d^2$ standard independent Wiener processes labeled $\{B^{k,l}_t\}_{1 \leq k,l \leq d}$. Finally, to fix the Brownian motion, we assume given an initial condition $W_0 \in \Mdd$.

\begin{theorem}[Menon, Yu~\cite{MY-dln}]
 \label{thm:BM-dln}
The solution $X^\beta_t$ to the following \Ito\/ SDE with initial condition $X^\beta_0 =W_0$  is Brownian motion on $(\Mdd,g^N)$ started at $W_0$:
\begin{equation}
    \label{eq:bm-dln}
        dX_t^\beta = \sqrt{\frac{2}{\beta}} \left( \begin{array}{lll}
    \sqrt{N}\lambda_1^{N-1} dB^{1,1}_t & \sqrt{\frac{\lambda_1^{2N}-\lambda_2^{2N}}{\lambda_1^2-\lambda_2^2}} dB^{1,2}_t & \ldots \\
    \sqrt{\frac{\lambda_2^{2N}-\lambda_1^{2N}}{\lambda_2^2-\lambda_1^2}} dB^{2,1}_t & \sqrt{N}\lambda_2^{N-1} dB^{2,2}_t & \ldots \\
    \vdots & \ddots & \vdots \end{array}\right) + \frac{1}{\beta} Q_N \Sigma'' Q_0^T \, dt.
\end{equation}
\end{theorem}
We may now state the RLE for the DLN in  analogy with equations~\eqref{eq:rle-up} and~\eqref{eq:rle-down}. The RLE for $\ww^{\beta,\kappa}$ is (cf. equation~\eqref{eq:big-grad1a})
\begin{equation}
    \label{eq:rle-dln-up}
d\ww^{\beta,\kappa} = -\nabla_\ww E(\phi(\ww))\, dt + d\bbb^{\beta,\kappa}.
\end{equation}
Then we show in~\cite{MY-dln} that the law of the end-to-end matrix $W_t=\phi(\ww_t)$ is given by the SDE
\begin{equation}
    \label{eq:rle-dln-down}
dW^{\beta,\kappa} = -\mathrm{grad}_{g^N} F_\beta(W^{\beta,\kappa})\, dt + dX^{\beta/\kappa},
\end{equation}
where $dX^{\beta/\kappa}$ is described by Theorem~\ref{thm:BM-dln} and  $F_\beta(W) = E(W)-\beta^{-1}S(W)$ where the entropy $S(W)$ is given by Theorem~\ref{thm:embed}. The gradient may be computed explicitly and we find
\begin{equation}
    \label{eq:rle-dln-down2}
\mathrm{grad}_{g^N} F_\beta(W^{\beta,\kappa}) = \mathcal{A}_{N,W}(E'(W)) - \frac{1}{\beta} Q_N \Sigma' Q_0^T.
\end{equation}
The second term on the right is the gradient of $-\beta^{-1} S(W)$. The expression for $\Sigma'$ in equation~\eqref{eq:sigma-prime} provides the analogue of Coulombic repulsion in the DLN. This `physical effect' is strictly due to the geometry of the DLN.

\section{Discussion}

\subsection{Summary}
\label{subsec:dis-overview}
This article has used the geometric theory of dynamical systems to study training dynamics in the DLN. Let us review the main ideas. 

Theorem~\ref{thm:ACH} and Theorem~\ref{thm:ACH2} allow  us to understand the foliation of $\Md^N$ by invariant manifolds, identify the important concept of balancedness and the dynamics on the balanced  manifolds. Theorem~\ref{thm:BRTW2} identifies the important role of  Riemannian geometry. Our presentation of these results differs from the original sources in that we use Riemannian geometry as the foundation for our analysis. 

Let us explain this shift more more carefully. By viewing Riemannian submersion as the fundamental paradigm, we discover an entropy formula for the DLN (Theorem~\ref{thm:entropy}) and obtain an explicit parametric description of the Riemannian manifolds $(\balance_r, \Fr)$ and $(\Mr,g^N)$ (Theorem~\ref{thm:submer} and Theorem~\ref{thm:embed}). 
We also obtain a thermodynamic framework, including an understanding of the underlying microscopic fluctuations as follows. The minimal intrinsic description of stochastic fluctuations is provided by the theory of Brownian motion on Riemannian manifolds. Our understanding of this abstract idea is guided by Dyson Brownian motion, a fundamental model in random matrix theory. In particular, the entropy formula is tied to Brownian motion on group orbits in analogy with Dyson Brownian motion.  Finally, in order to extend gradient descent to stochastic gradient descent in a geometrically faithful manner, we introduce and compute Riemannian Langevin equations for the DLN. We stress the importance of an anisotropic splitting of the noise between `noise in the gauge' and `noise in the observable', in order to provide a geometric framework for the thermodynamics of deep learning.

As noted at the outset, while all the Theorems in this article have a classical feel, the results presented here have been obtained very recently. As a dynamical system, the DLN is fascinating for the range of concrete, tractable questions that arise from a careful examination of overparametrization. In particular, the connections with the theory of minimal surfaces and random matrix theory, suggest new questions on motion by curvature and the large $d$ and large $N$ asymptotics of the  DLN. Some of these are summarized in Section~\ref{sec:open}. 

\subsection{Linear networks and deep learning}
\label{subsec:dln-dl}
A good model must provide insights of practical importance, while at the same time being tractable to to a detailed mathematical analysis. As we have seen, the analysis of training dynamics in the DLN can be studied in rich detail. But how good is it as a guide to deep learning in practice? Let us now present some ideas explored in the literature on linear networks. We then return to the heuristics of deep learning listed in Section~\ref{sec:deep-learning}, contrasting these with the DLN.

{Phenomenological linear models have been explored in the neural network literature since the 1990s. The energy landscape for linear networks was studied carefully by Baldi and Hornik~\cite{Baldi-Hornik1,Baldi-Hornik2}. Since the advent of deep learning, a central concern has been to use the DLN to understand implicit bias, to understand the role of depth, and to understand the energy landscape for several matrix learning tasks. This has given rise to a large literature on linear networks, of which we can only present a few samples.}

An important result in a similar spirit to~\cite{ACGH,ACH}, especially a characterization of the minimizer when the depth $N=2$ and $W_2=W_1^T$ (the Bures-Wasserstein reduction of the DLN) is given in~\cite{Gunasekar2018,Gunasekar2017}. There have also been several studies of the nature of critical points for matrix learning tasks and the related convergence theory in the DLN for different choices of loss functions~\cite{Brechet-Montufar,Ge1,Rauhut2}. An early use of exact solutions in DLN, as well as the use of the DLN to develop a theory for semantic development is presented in~\cite{Saxe2,Saxe1,Saxe3}.

Linear {\em convolutional\/} networks (LCNs) were introduced to shed light on deep learning with convolutional neural networks.  An algebraic geometric analysis of LCNs has provided a detailed understanding of the nature of the singularities of the function spaces as a function of the network architecture~\cite{Montufar1,Montufar2}. In these studies, linear networks are used to shed light on the expressivity of neural networks. A more recent direction is to include nonlinearity that respects the matrix geometry by depending only on the singular values ~\cite{ghosh}.

\subsection{Numerical experiments.}
Like most models in machine learning, our understanding of the DLN relies heavily on numerical experiments. These experiments reveal subtle phenomena, such as the importance of low-rank matrices as $\lim_{t\to \infty} W(t)$, and the importance of the balanced manifolds. Let us explain these in turn.

Numerical experiments in~\cite{ACGH,ACH,ACHL,CMV} explore several aspects of the DLN, including the attraction to low-rank matrices. For example, in a numerical study of matrix completion in~\cite{CMV} we observe that the effective rank of $W(t)$ decreases in time through sudden jumps; roughly the effective rank stays constant for a long period and suddenly drops by $1$. We also find an asymptotic distribution of limits $\lim_{t\to \infty} W(t)$ on the low-rank manifolds. 

We may design similar numerical experiments for deep learning as follows. We fix a family of random functions $\mathcal{F}$ (e.g. random Fourier series on the line with a fixed decay rate) and then numerically observe the training dynamics for a neural network approximation to a randomly chosen $f \in \mathcal{F}$. It is easy to vary the network architecture (depth, width, neural unit) and study the resulting variation in the training dynamics. Experiments of this nature have been carried out in classroom projects designed by the author. They reveal that typical neural networks fit functions scale-by-scale, with jumps in scale (i.e. sudden increases in oscillations) that are analogous to the jumps in rank observed in the DLN. These experiments lead us to conjecture that the fidelity between the DLN and deep learning is much closer than what may be naively expected. 

These are numerical experiments, not theorems. But the nature of the transient dynamics, not just the time asymptotics, is a fertile area of enquiry where we believe it is possible to transfer insights between the DLN and deep learning.  In particular, the training dynamics in deep learning also suggest an interplay between curvature and entropy, corresponding to fluctuations in the null directions. We hope that the results in this article will lead to methods to quantify this entropy using functional analytic tools. Any such analysis must be matched with numerics; it is necessary to continuously design numerical experiments, in parallel with the refinement of the theory of DLN, to create an intuition for deep learning.

\subsection{Balancedness}
Let us now discuss aspects of the balanced manifolds that remains somewhat mysterious. That the balanced variety has a beautiful mathematical structure is not in doubt, but the invariant manifold theorems don't really explain why the balanced manifolds should be as important in practice as they are. Since the balanced variety is just the $\mathbf{G}=\mathbf{0}$ special case of the $\mathbf{G}$-balanced varieties, the odds that a randomly chosen initial condition will lie on $\balance$ is zero! 

There have been extensions of the idea of balancedness to networks that vary linearity locally~\cite{JasonLee}. More recently, the geometry of the vector fields in the determination of conservation laws has been considered in~\cite{Peyre1,Peyre2}. These results are promising, but do not yet constitute a satisfactory understanding of balancedness.

Our introduction of the RLE in Section~\ref{sec:RLE}, especially the framework of quasistatic thermodynamic equilibration and the RLE for the DLN is intended to shed light on this question. While these theorems are restricted to the balanced manifold $\balance$ of full rank, we may define analogous RLE `upstairs' on any $\balance_\mathbf{G}$. The underlying conjecture is that small noise, for example due to round-off error, may be modeled by such RLE and that in the $\beta>0$, $\kappa=0$ regime, we may observe motion by curvature of 
$\balance_\mathbf{G}$ towards $\balance_{\mathbf{0}}$. We hope that this approach will explain the mysterious appearance of the Simons cone in the DLN. 

In Section~\ref{sec:deep-learning} we noted that formally Riemannian submersion also applies to deep learning. What is still missing in this analogy is the concept of balancedness. That is, while there is no ambiguity about the importance of balancedness for the DLN, we don't yet know what form this idea must take for deep learning. The parametrization by $\mathfrak{z}$ and Theorem~\ref{thm:embed} reveal a subtle coupling along the depth of the network on $\balance$. A natural conjecture is that balancedness must correspond to local equilibrium, i.e. {\em detailed balance\/}, along the network. However, the specific mathematical form of this idea remains elusive. 

\section{Open problems}
\label{sec:open}
We begin with some concrete questions that do not require much machinery. This is followed by a broader discussion of perspectives to explore. 

\subsection{Convergence to a low-rank matrix}
\label{subsec:entropy-selection}
Here is a simply stated problem. Fix $d=2$ and 
consider matrix completion for $W$ when given that $W_{11}=W_{22}=1$. We study this problem with the DLN as follows: we use equation~\eqref{eq:mc1a} to define the quadratic energy function
\begin{equation}
    \label{eq:mc2}
    E(W) = \frac{1}{2}\left( (W_{11}-1)^2 + (W_{22}-1)^2 \right).
\end{equation}
Observe that $E$ vanishes for any matrix of the form
\begin{equation}
    \label{eq:mc3}
    W = \left( \begin{array}{ll} 1 & * \\ * & 1 \end{array}\right).
\end{equation}
Further, there is a variety of rank-one solutions of the form
\begin{equation}
    \label{eq:mc4}
    W = \left( \begin{array}{cc} 1 & a \\ \frac{1}{a} & 1 \end{array}\right), \quad a \in \R\backslash\{0\}.
\end{equation}
This situation is illustrated in Figure~\ref{fig:geom1}.
Numerical experiments (see~\cite[Figures 3--6]{CMV}) reveal that the solution to~\eqref{eq:closed3} with random initial condition  $W_0 \in \mathbb{M}_2$ is attracted to the low-rank solutions. Further, these solutions cluster around the matrix 
\begin{equation}
    \label{eq:mc5}
    W = \left( \begin{array}{cc} 1 & 1 \\ 1 & 1 \end{array}\right), \quad a \in \R\backslash\{0\}.
\end{equation}
A rigorous analysis of this example, including the $N$ dependence, would provide considerable new insight. Specifically, we seek the resolution of the following question. Consider the dynamical system
\begin{equation}
\label{eq:repeat}
\dot{W} = -\sum_{k=1}^N (WW^T)^{\tfrac{N-k}{N}} E'(W) (W^TW)^{\tfrac{k-1}{N}},
\end{equation}
with $E(W)$ given by equation~\eqref{eq:mc2}. Can one establish convergence of the solution $W(t)$ for suitable initial conditions to the low-rank energy minimizer in equation~\eqref{eq:mc3} as $t \to \infty$? Are there initial conditions with $E(W)>0$ whose solutions converge to other minimizers of $E$?

This problem illustrates a fundamental issue in dynamical systems theory. While one may have theorems that guarantee convergence, the precise identification of the limit is not straightforward even in apparently simple problems. 

\subsection{The free energy landscape}
We return to a theme discussed in Section~\ref{subsec:entropy}. Does the inclusion of entropy provide a selection principle, especially for degenerate loss functions?

For each depth $N$, we may consider the entropy defined in Theorem~\ref{thm:entropy}, and the free energy defined in equation~\eqref{eq:free-energy}. An interesting aspect of this free energy is that while the entropy depends on the singular values $\Sigma$ alone, the energy depends on all of $W$. This makes the computation of the set of minimizers an interesting problem, even when $d=2$.

The simplest concrete problem in this class is as follows. Consider again the energy in equation~\eqref{eq:mc2}. What is the nature of the minimizing set $\mathcal{S}_\beta$? Specifically, how does it vary with $\beta$ and $N$?

\subsection{Large $d$ and large $N$ asymptotics}
An important feature of the DLN is that it admits a natural $N\to \infty$ limit. This limit requires that we rescale $\aaa_{N,W}$ defined in equation~\eqref{eq:metric2} as follows
\begin{equation}
    \label{eq:large-N1}
    \frac{1}{N} \sum_{p=1}^N (WW^T)^{\frac{N-p}{N}} \, Z \, (W^TW)^{\frac{p-1}{N}},
\end{equation}
and let $N\to \infty$ to obtain the limiting operator
\begin{equation}
    \label{eq:large-N2}
    \aaa_{\infty,W} (Z)=\int_0^1 (WW^T)^{1-s} \, Z \, (W^TW)^{s} \, ds.
\end{equation}
This operator and the resulting metric $g^\infty$ were studied in~\cite{CMV}. In particular, the metric may be diagonalized as in Section~\ref{subsec:metric}, the crucial lemma being 
\begin{lemma}
    \label{le:diagon-infty}
    The operator $\aaa_{\infty,W}: T_W \Mdd \to T_W \Mdd$ is symmetric and positive definite with respect to the Frobenius inner-product. It has the $d^2$ eigenvalues and eigenvectors 
    \begin{equation}
        \label{eq:diagon-infty}
        \aaa_{\infty,W} \, u_k v_l^T = \frac{\sigma_k^2-\sigma_l^2}{2 \log \sigma_k/\sigma_l} \, u_k v_l^T , \quad 1\leq k,l \leq d,
    \end{equation}
when $k \neq l$ and 
\begin{equation}
        \label{eq:diagon-infty-b}
        \aaa \, u_k v_k^T = \sigma_k^{2}\, u_k v_k^T , \quad 1\leq k \leq d.
\end{equation}
\end{lemma}

Recall that $u_k$ and $v_l$ define the normalized left and right singular vectors of $W$.   The analog of equation~\eqref{eq:p-metric1} is
\begin{equation}
    \label{eq:p-metric2}
       g^\infty(Z,Z) = \sum_{1 \leq k \leq d} \frac{1}{\sigma_k^2}\, Z_{kk}^2 + \sum_{1\leq k, l \leq d, k \neq l} 
       \frac{2\log\sigma_k/\sigma_l}{\sigma_k^2-\sigma_l^2}
       \, Z_{kl}^2,
\end{equation}
where we have written $Z \in T_M\Mdd$ as $Z= Z_{kl}u_k v_l^T$.

The rescaling of the metric affects the entropy by a constant factor, but since it is only the {\em gradient\/} of the entropy that matters we also have the limit 
\begin{equation}
    \label{eq:entropy-infinity}
    S_\infty(W) = \sum_{1 \leq k < l \leq d} \log \left( \frac{\sigma_k^2 -\sigma_l^2}{2 \log \sigma_k/\sigma_l} \right). 
\end{equation}

These formulas are reminiscent of determinantal formulas in random matrix theory (see~\cite{Baik-Deift-Suidan} for examples). Their large $d$ asymptotics have not been studied. Further, the analogy between Dyson Brownian motion and the RLE for DLN in Section~\ref{subsec:RLE-dln} suggest that the large $d$ asymptotics for the equilibrium distribution should be described by universal fluctuations. In order to make such a study precise, we suggest the reader first examine the equilibrium distribution of equation~\eqref{eq:rle-dln-down} for the simplest quadratic energy $E(W)=\tfrac{1}{2}\Tr(W^TW)$, establishing the analog of the semicircle law for the DLN. This may be followed by studies of fluctuations.

An interesting facet of the large $N$ limit is that while the formulas for the metric and entropy simplify since these rely only on $W$ downstairs, there is no longer a space upstairs! Is there a limiting framework for such Riemannian submersion?

\subsection{Low-rank varieties}
The study of the DLN would benefit greatly from a more careful examination of the underlying algebraic and differential geometry. Let us focus on the influence of rank to illustrate this point.

All the main ideas in this article were first established for matrices of full rank. While this is a convenient choice, a careful analysis of the foliations by rank is necessary for both theoretical and practical purposes.

First, as we have remarked above, in practice low-rank matrices seem to appear as $t\to \infty$ limits, even when the dynamics is restricted to $W \in \Mdd$. Even for simple energies, such as in matrix completion, it would be very useful to understand how the set of minimizers may be arranged by rank. 

Second, our results on Riemannian submersion and RLE, typically require smoothness of the submersion map $\phi$ and the assumption of distinct singular values. At present, these are convenient assumptions that allow us to prove interesting theorems. However, an understanding of the singularities that may arise at repeated singular values, especially repeated zero singular values as in the low-rank setting, is necessary for a comprehensive understanding of the model.

\subsection{Coexistence of gradient and Hamiltonian structures}
 \label{rem:ach2}
 Gradient flows typically have a character that is complementary to Hamiltonian systems. There are, however, rare cases of dynamical systems that are both gradient-like and completely integrable Hamiltonian~\cite{Bloch-Brockett-Ratiu}. We expect that a deeper examination of the DLN will reveal similar structure for the following reasons.

Training dynamics in the DLN are given by a gradient flow. 
 However, the theorems in this article also suggest a complementary symplectic geometry. Here we note that the natural geometric setting for a Hamiltonian system is that of  a symplectic manifold. Further, co-adjoint orbits of Lie groups constitute one of the fundamental examples of symplectic manifolds, including most known completely integrable systems. In the analysis of the DLN, we see that the existence of the conserved quantities $\mathbf{G}$ -- a typical feature of integrable Hamiltonian systems -- is tied to the symmetries of overparametrization as in the theory of Hamiltonian systems. Further, the appearance of determinantal formulas as in random matrix theory is also suggestive of completely integrability. 

\section{Acknowledgements}
Sanjeev Arora introduced me to the DLN (and Nadav Cohen!)  at IAS in 2019. Little did I realize at the time how rewarding the study of the DLN would be, so many thanks Sanjeev. The results presented here are all joint work with Nadav Cohen, Zsolt Veraszto, and Tianmin Yu. They also build on related projects  with Dominik Inauen, Ching-Peng Huang, Tejas Kotwal, Jake Mundo and Lulabel Ruiz-Seitz. I am particularly grateful to Nadav for many patient explanations about the nature and promise of deep learning and the pleasure of joint work.

This article was developed on the basis of minicourses at Kyushu University, CIMAT, the National University of Singapore and the Palazzone di Cortona of the Scuola Normale Superiore, Pisa. Each lecture provided an opportunity to synthesize several results on the DLN for a diverse audience, improving the exposition in each iteration. I express my thanks to Andrea Agazzi (Pisa), Octavio Arizmendi Echegaray (CIMAT), Carlos Pacheco (Cinvestav), Subhro Ghosh (NUS), Philippe Rigollet (MIT) and Tomoyuki Shirai (Kyushu) for providing me with the opportunity to share these ideas and for many stimulating discussions. I am especially grateful to the students at CIMAT and Cortona for their warmth and enthusiasm. 

Related discussions with Jos\'{e} Luis P\'{e}rez Garmendia (CIMAT), Alex Dunlap, Jian-Guo Liu and Jonathan Mattingly (Duke), Guido Montufar (UCLA), Austin Stromme (ENSAE/CREST), Praneeth Netrapalli (Google), Courtney and Elliott Paquette (Montreal), Qianxiao Li and Xin Tong (NUS), Boris Hanin, Pier Beneventano and Mufan Li (Princeton),  Jamie Mingo (Queens University),  Eulalia Nualart (UPF, Barcelona),   Noriyoshi Sakuma (Nagoya) and Sinho Chewi and Zhou Fan (Yale) have improved this work.

I am a dynamicist, not an expert in deep learning, and it has been a real thrill to learn from computer scientists and statisticians.  I hope that this article communicates the spirit of the geometric theory of dynamical systems in a way that is useful to practitioners, reciprocating the gift of a beautiful model.

\bibliographystyle{siam}
\bibliography{dln}

\end{document}